\documentclass{article}

\usepackage{microtype}
\usepackage[pdftex]{graphicx}
\usepackage{wrapfig}
\usepackage{subcaption}
\usepackage{booktabs} 

\usepackage[breaklinks,hyperfootnotes=false]{hyperref}

\usepackage[accepted]{icml2024}

\usepackage{amsmath}
\usepackage{amssymb}
\usepackage{mathtools}
\usepackage{amsthm}

\usepackage[capitalize,noabbrev]{cleveref}

\usepackage{bbm}
\usepackage{nicefrac}
\usepackage{mathtools}
\usepackage{setspace}
 
\usepackage{amsmath,amsfonts,bm}

\def\eqref#1{equation~\ref{#1}}

\def\1{\bm{1}}

\DeclareMathAlphabet{\mathsfit}{\encodingdefault}{\sfdefault}{m}{sl}
\SetMathAlphabet{\mathsfit}{bold}{\encodingdefault}{\sfdefault}{bx}{n}

\def\gB{{\mathcal{B}}}
\def\gC{{\mathcal{C}}}
\def\gD{{\mathcal{D}}}

\def\gO{{\mathcal{O}}}

\def\gS{{\mathcal{S}}}

\def\gX{{\mathcal{X}}}
\def\gY{{\mathcal{Y}}}

\newcommand{\E}{\mathbb{E}}

\newcommand{\R}{\mathbb{R}}

\DeclareMathOperator*{\argmax}{arg\,max}

\DeclareMathOperator*{\argmin}{arg\,min}

\DeclareMathOperator{\sign}{sign}

\usepackage{thm-restate}
\usepackage{url}
\usepackage{adjustbox}

\usepackage{listings}
\usepackage{colortbl}
\usepackage{float}
\usepackage{xspace}
\usepackage{multirow}
\usepackage{enumitem}
\setlist{leftmargin=4mm,itemsep=0pt,topsep=0pt}
\usepackage{selectp}

\makeatletter
\def\blfootnote{\gdef\@thefnmark{}\@footnotetext}
\makeatother

\colorlet{alternateRowColor}{magenta!10}

\newcommand{\norm}[1]{\ensuremath{\| #1 \|}}
\newcommand{\twonorm}[1]{\norm{#1}_2}
\newcommand{\normsq}[1]{\ensuremath{\norm{#1}^2}}
\newcommand{\twonormsq}[1]{\ensuremath{\twonorm{#1}^2}}
\newcommand{\matrixnorm}[2]{\ensuremath{\norm{#1}_{#2}}}
\newcommand{\matrixnormsq}[2]{\ensuremath{\normsq{#1}_{#2}}}

\renewcommand{\P}{\mathbb{P}}

\newcommand{\calN}{\mathcal{N}}
\newcommand{\calR}{\mathcal{R}}

\renewcommand{\P}{\mathbb{P}}
\newcommand{\posutil}{u}
\newcommand{\costmatrix}{\Sigma}
\newcommand{\costset}{\gC}
\newcommand{\costfunc}{c}
\newcommand{\zoloss}{\ell_{0-1}}
\newcommand{\hingeloss}{\ell_{\textrm{hinge}}}
\newcommand{\shingeloss}{\ell_{\textrm{s-hinge}}}

\newcommand{\shingeemprisk}{\hat R_{\textrm{s-hinge}}}

\newcommand{\shingepoprisk}{R_{\textrm{s-hinge}}}

\newcommand{\uinv}{\ensuremath{u_{*}}}
\newcommand{\dualnorm}[1]{\|#1\|_{*}}
\newcommand{\psdnorm}[2]{\matrixnorm{#1}{#2}}
\newcommand{\psdnormsq}[2]{\matrixnormsq{#1}{#2}}
\newcommand{\dualpsdnorm}[2]{\psdnorm{#1}{*,#2}}
\newcommand{\mincostscale}[1]{\sigma_{#1\ell}}
\newcommand{\maxcostscale}[1]{\sigma_{#1u}}
\newcommand{\gYhat}{{\hat{y}}}
\newcommand{\cinc}{{\costfunc\in\costset}}

\newtheorem{theorem}{Theorem}[section]

\newtheorem{lemma}[theorem]{Lemma}
\newtheorem{corollary}[theorem]{Corollary}
\newtheorem{proposition}[theorem]{Proposition}

\theoremstyle{definition}

\newtheorem{definition}[theorem]{Definition}

\renewcommand{\paragraph}[1]{{\textbf{#1}\,\,}}

\begin{document}

\twocolumn[
\icmltitle{One-Shot Strategic Classification Under Unknown Costs}



\icmlsetsymbol{equal}{*}

\begin{icmlauthorlist}
\icmlauthor{Elan Rosenfeld}{cmu}
\icmlauthor{Nir Rosenfeld}{tech}
\end{icmlauthorlist}

\icmlaffiliation{cmu}{Carnegie Mellon University}
\icmlaffiliation{tech}{Technion -- Israel Institute of Technology}

\icmlcorrespondingauthor{Elan Rosenfeld}{elan@cmu.edu}

\icmlkeywords{Machine Learning, ICML}

\vskip 0.3in
]



\printAffiliationsAndNotice{}  

\begin{abstract}
The goal of strategic classification is to learn decision rules which are robust to strategic input manipulation. Earlier works assume that these responses are known; while some recent works handle unknown responses, they exclusively study online settings with repeated model deployments. 
But there are many domains---particularly in public policy, a common motivating use case---where multiple deployments are infeasible, or where even one bad round is unacceptable.
To address this gap, we initiate the formal study of \emph{one-shot} strategic classification under unknown responses, which requires committing to a single classifier once. Focusing on uncertainty in the users' cost function, we begin by proving that for a broad class of costs, even a small mis-estimation of the true cost can entail trivial accuracy in the worst case. In light of this, we frame the task as a minimax problem, aiming to minimize worst-case risk over an uncertainty set of costs. We design efficient algorithms for both the full-batch and stochastic settings, which we prove converge (offline) to the minimax solution at the rate of $\tilde{\mathcal{O}}(T^{-\nicefrac{1}{2}})$. Our analysis reveals important structure stemming from strategic responses, particularly the value of \emph{dual norm regularization} with respect to the cost function.\looseness=-1
\end{abstract}

\section{Introduction}

Across a multitude of domains, 
machine learning is increasingly being used to inform decisions about human users.
But when users stand to gain from certain predictive outcomes,
they may act to obtain favorable predictions by modifying their features. 
Since this can harm predictive performance,
learning becomes susceptible to \emph{Goodhart's law},
which states that ``when a measure becomes a target, it ceases to be a good measure'' \citep{goodhart1984problems}.
This natural tension between learning systems and their users applies broadly: loan approvals, admissions, hiring, insurance, and welfare benefits are all examples in which the system seeks predictions that are accurate, whereas users---irrespective of their true label---wish to be classified as positive.

The field of \emph{strategic classification} \citep{bruckner2012static,hardt2016strategic}
studies learning in such settings, with the principal aim of learning classifiers that are robust to strategic user behavior.
However, most works in this field rely on the key assumption that the learner knows precisely how users would respond to any given classifier.
This is typically modeled as knowledge of the underlying \emph{cost function} $\costfunc(x,x')$
which defines the cost users incur for 
modifying features $x$ to become $x'$.
This assumption enables tractable learning, but it is unrealistic and it effectively 
takes the ``sting'' out of Goodhart's law:
in a statistical sense,
the power to anticipate user responses completely nullifies the effect of users' gaming behavior on predictive performance (for radial basis cost functions; see \cref{apx:in_principle}).

Indeed, if we survey some well-known examples of Goodhart's law \citep[e.g.][]{chrystal2003goodhart,elton2004goodhart,fire2019over,teney2020value}, it becomes apparent that much of the policy challenge arises precisely from
\emph{not} knowing how users would respond.
This motivates us to instead focus on learning strategically robust classifiers under an \emph{unknown} cost function.
To cope with this uncertainty, we take a conservative approach and model the system as aiming to learn a classifier that is robust to all cost functions $\costfunc$ in some uncertainty set $\costset$,
which we think of as including all cost functions which are believed to be plausible (alternatively, it is a set which we can be confident will contain the true, unknown $\costfunc$).
Our approach is therefore doubly robust, providing guarantees under both the manipulation of inputs by strategic behavior and an adversarial choice of cost function.
We argue this is necessary:
we prove 
that if one optimizes for strategic classification under a single, fixed cost, then any discrepancy between that cost and the true cost
can result in dramatically reduced accuracy.

While some works have studied strategic learning under unknown responses
\citep{dong2018strategic,ahmadi2021strategic,shao2023strategic,lechner2023strategic,harris2023strategic},
their focus is entirely on sequential learning settings.
This allows them to cope with response uncertainty via exploration:
deploying a series of models over time,
observing how users respond, and adapting accordingly.
Algorithms for such online settings are designed to ensure that regret decays sufficiently fast with the number of rounds---but they provide no guarantees on worst-case outcomes for any single deployment.
We observe that there are many realistic settings
in which multiple deployments are too costly or technically impossible
(e.g. financial regulation);
in which arbitrary exploration is unrealistic or unethical
(e.g. testing in education);
in which there is need for immediately beneficial short-term outcomes
(e.g. epidemic vaccination);
or in which even a single bad round could be very harmful
(e.g. environmental conservation).

Motivated by such examples,
this work studies robust strategic learning
in a ``one-shot'' setting where the learner must commit to one single model at deployment.
In this setting, we recast our objective as optimizing for a classifier that minimizes the worst-case strategic risk over an uncertainty set $\costset$.
We devise two efficient algorithms, 
for the full-batch and stochastic settings, which provably converge to the minimax solution at the rate $\tilde{\mathcal{O}}(T^{-\nicefrac{1}{2}})$.
Notably, this rate is dimension independent for the typical cost functions---including $\ell_1$ and $\ell_2$ norms---and it is achieved despite the inner maximization being non-concave.
A key step in our approach is to adapt the \emph{strategic hinge loss} \citep{levanon2022generalized} to properly handle a large set of possible costs beyond the $\ell_2$ norm. As the strategic hinge loss is non-convex, we apply a regularization term that accounts for the unique structure of strategic response: our analysis uncovers that the ``correct'' form of regularization is in fact the \emph{dual norm of $\beta$ with respect to the cost}. We complement this with an updated generalization bound for this loss which corrects an error in that previous work.

More broadly, our approach relies on the observation that strategic responses
induce a shift in the data distribution \citep{perdomo2020performative}.
Because different costs induce different shifts,
uncertainty over cost functions translates to uncertainty over test distributions.
This allows us to formulate our problem as one of \emph{distributionally robust optimization} (DRO) \citep{namkoong2016stochastic,duchi2021learning},
where the goal is to predict well on the worst-case distribution within some given set.
Whereas the typical DRO formulation defines the uncertainty set with respect to a more traditional probability divergence,
in our case the set of distributions is inherited from the structure of strategic responses
and includes all shifts that can result from strategic behavior under any $\costfunc \in \costset$.
This means that the entire \emph{set} of possible distributions becomes dependent on the classifier,
which is one of the primary challenges our work addresses.


\subsection{Related work} \label{sec:related}


\paragraph{Strategic classification.}
Introduced in \citet{hardt2016strategic},
and based on earlier works
\citep{bruckner2009nash,bruckner2012static,grosshans2013bayesian},
the literature on strategic classification has since been growing steadily.
Focusing on supervised classification in the batch (offline) setting,
here we list a relevant subset.
Advances have been made on both statistical
\citep{zhang2021incentive,sundaram2021pac}
and algorithmic aspects of learning
\citep{levanon2021strategic},
but the latter lacks guarantees.
In contrast, our work provides efficient algorithms that are provably correct.
Efforts have also been made to extend learning beyond the basic setting of \citet{hardt2016strategic}.
These include:
accounting for users with
noisy estimates \citep{jagadeesan2021alternative},
missing information
\citep{ghalme2021strategic,bechavod2022information},
more general preferences
\citep{sundaram2021pac,levanon2022generalized, eilat2023strategic};
incorporating causal elements into learning
\citep{miller2020strategic,chen2023linear,horowitz2022causal,mendler2022anticipating},
and considering societal implications
\citep{milli2019social,levanon2021strategic,lechner2021learning}.

\paragraph{(Initially) unknown user responses.}
Several works have considered the case of inferring unknown user responses under different forms of strategic learning,
but in online or sequential settings.
\citet{dong2018strategic} bound the Stackelberg regret of online learning when both costs and features are adversarial,
but when only negatively-labeled users respond; \citet{harris2023strategic} bound a stronger form of regret when responses only come from positively-labeled users.
\citet{ahmadi2021strategic} propose a strategic variant of the perceptron and provide mistake bounds for $\ell_1$ norm with unknown diagonal scaling or $\ell_2$ norm multiplied by an unknown constant.
\citet{shao2023strategic} study learning under unknown ball manipulations with personalized radii,
and give mistake bounds and (interactive) sample complexity bounds for different informational structures.
\citet{lechner2023strategic} bound the sample and iteration complexity of learning under general, non-best response manipulation sets via repeated model deployments to infer the manipulation graph.
\citet{lin2023plug} learn under a misspecified response model,
inferred in a (nonadaptive) exploration phase from some class of models.
The analyses in these works better reflect reality in that the exact response cannot be known. But 
online deployment and long-term regret minimization are not appropriate for many natural use cases of strategic classification, which motivates our investigation of the one-shot setting.

\section{Preliminaries}
\paragraph{Notation.}
We study the problem of linear strategic classification on a distribution over inputs $x \in \gX\subseteq\R^d$ and labels $y\in\gY=\{\pm 1\}$ where the exact response of the test population is unknown.
We consider linear classifiers $\sign(\beta^\top x)$,
optimized over some bounded set $\gB\subset \R^{d+1}$
(this includes a bias term which we leave implicit and which is not included in the vector norm).
To model users' responses, we consider the typical setup of a \emph{cost function} $\costfunc(x, x')$ which defines the cost for an agent to change their features from $x$ to $x'$. Together with a utility $\posutil \geq 0$ gained from a positive classification, this cost determines the strategic response of a rational agent to a classifier $\beta$ via
\begin{align}
    \label{eq:best_response}
    x(\beta) := \argmax\nolimits_{x'} \left[ \mathbf{1}\{\beta^\top x' \geq 0\} \cdot u - c(x,x') \right].
\end{align}
Thus, an agent will move only if it would change their classification from negative to positive, and only if $c(x, x') < \posutil$. We handle non-uniqueness of the argmax by breaking ties arbitrarily, but we follow convention by assuming no strategic response if the net utility is exactly 0. We let $\delta(x; \beta) := x(\beta) - x$ denote the movement of an agent, and we suppress dependence of $\delta$ on $x, \beta$ where clear from context.
\looseness=-1

\paragraph{Form of the cost function.} We study cost functions that can be written as $\costfunc(x, x') = \phi(\norm{x' - x})$ for a norm $\norm{\cdot}$ and non-decreasing function $\phi:\R_{\geq 0}\to\R_{\geq 0}$.
This includes the $\ell_2$-norm (by far the most common cost, sometimes squared), but it is also much more general---it allows for different non-linear transformations which may better reflect real-world costs, such as $\phi(r) = \ln (1+r)$, and we allow $\norm{\cdot}$ to denote any differentiable and monotonic norm, which includes all $p$-norms (where $p\in[1,\infty]$).
This significantly expands upon the costs discussed in the literature thus far.
\looseness=-1

\paragraph{Parameterizing the space of costs.}
Recall that we are interested in robust prediction when we cannot know the exact strategic response.
Keeping the assumption of rational behavior, this naturally points to an unknown cost function as a primary source of this uncertainty.
We parameterize this uncertainty via an unknown positive definite (PD) matrix $\costmatrix \succ \mathbf{0}$ in $\R^{d\times d}$ which scales the relative costs per input dimension, denoting the induced norm as $\psdnorm{\cdot}{\costmatrix}$.
For the 2-norm, this is the standard PD norm given by
$\psdnorm{x}{\costmatrix} := \sqrt{x^\top \costmatrix x}$,
but we define it for general norms as $\psdnorm{x}{\costmatrix} := \norm{\costmatrix^{1/2} x}$. Hence, any cost function $\costfunc$ is uniquely determined by its parameterization $\costmatrix$ as $c(x, x') = \phi\left(\|\Sigma^{1/2}(x'-x)\|\right)$; we will use these two variables interchangeably. The other two factors in strategic response are the positive utility $u$ and the monotone transform $\phi$: for reasons that will soon be made clear, we only require knowledge of the maximum value $\uinv$ such that $\phi(\uinv) \leq u$ (this is clearly satisfied for known $\phi$ and $u$, as is typically assumed).
Note this value need not be shared among users, and it suffices to know an interval in which it lies---but for simplicity we treat it as fixed. 

\paragraph{Encoding uncertainty.}
Since the true cost is not known,
and since we cannot estimate it in an online fashion, we instead assume a system-specified \textit{uncertainty set} $\costset$, defined as a compact, convex set of possible costs $\costfunc$ which is expected to contain the true cost. The goal of our analysis will be to derive strategies for efficiently identifying a classifier which ensures optimal (and boundable) worst-case test performance over \emph{all} costs $\costfunc \in \costset$, and therefore also bounded error under the true cost. Notably, this also means that \emph{even if the true cost changes over time}, our error bound will hold so long as the cost remains within $\costset$. In practice we want $\costset$ to be broad enough that we can be confident it contains the 
test-time cost%
---but we will also see that our convergence guarantees scale inversely with the diameter of this set, so it should be selected to be no larger than necessary.

\subsection{Strategic Learning Under a Single, Known Cost}
As a first step towards learning robustly under all costs in $\costset$, it will be useful to first consider learning under a single fixed cost.
For a \emph{known} cost $c$, the typical goal would be to minimize the 0-1 loss under strategic response with this cost: $\zoloss^c(\beta^\top x, y) := \mathbf{1}\{\sign(\beta^\top (x+\delta)) \neq y\}$. It is common to instead consider a more easily optimized surrogate such as the hinge loss $\hingeloss(\beta^\top x, y) := \max\{0, 1-y\beta^\top (x+\delta)\}$, but the discontinuous nature of $\delta(x; \beta)$ w.r.t. $\beta$ means that optimization is still intractable. Instead, we make use of the recently proposed \emph{strategic hinge loss} \citep{levanon2022generalized} 
which augments the standard hinge loss with an additional term to account for this discontinuity:
$\shingeloss(\beta^\top x, y) := \max\{0, 1-y(\beta^\top x + 2\twonorm{\beta} )\}$ (note this does not explicitly include $\delta$).

Unfortunately, even this relaxation poses difficulties. Firstly, the objective is non-convex---though \citet{levanon2022generalized} show that for known costs it often learns reasonable classifiers in practice,
once we transition to unknown costs it will become clear that having a guaranteed sub-optimality bound, which requires convexity, is important.
Second, the additional term $2\twonorm{\beta}$ captures the effect of strategic response only under the standard $\ell_2$-norm cost.\footnote{While \citet{levanon2022generalized} do discuss more general costs, they give only a generic description.}

\paragraph{Cost-aware strategic hinge.}
For our setting, we derive a more general strategic hinge loss that applies to the broader class of costs. This loss admits a natural form which relies on the \emph{dual norm} of $\beta$ with respect to the cost function:
\begin{definition}
    The \emph{$\costmatrix$-transformed dual norm} of $\beta$ is denoted $\dualpsdnorm{\beta}{\costmatrix} := \sup_{\psdnorm{v}{\costmatrix}=1} \beta^\top v = \dualnorm{\costmatrix^{-\nicefrac{1}{2}} \beta}$. We may leave dependence on $\costmatrix$ implicit, writing simply $\dualnorm{\beta}$.
\end{definition}
\begin{definition}
    The \emph{cost-dependent strategic hinge loss} is: 
    \begin{align} \label{eq:cost-dep_stratgic_hinge}
        \shingeloss^c(\beta^\top x, y) &:= \max\{0, 1 - y(\beta^\top x + \uinv \dualnorm{\beta})\}.
    \end{align}
\end{definition}

Note our proposed $\shingeloss^c$ generalizes the previous $\shingeloss$ since the $\ell_2$-norm is its own dual. The appearance of the dual norm here is not by chance. This quantity captures the dimension-wise sensitivity of our decision rule to changes in $x$, scaled \textit{inversely proportionally} to the cost an agent incurs for moving in that dimension. This should be intuitive: for any given direction, the less it costs a user to modify their features, the more they can afford to move, and thus the greater the importance of reducing our classifier's sensitivity to it. We make this formal with the following lemma
which bounds the maximal strategic change to inputs.
\begin{lemma}
    \label{lemma:dualnorm-soln}
    Fix $\beta$ and let $c(x, x') = \phi(\psdnorm{x-x'}{\costmatrix})$. The maximum possible change to a user's score 
    due to
    strategic behavior is
    $\uinv \dualpsdnorm{\beta}{\costmatrix}$.
\end{lemma}

Proofs of all lemmas can be found in \cref{appsec:main-lemmas-proofs}. Thus we see how augmenting the hinge loss with the dual norm serves as a natural approach to robust strategic classification, and we use this loss throughout. More generally, it will also be useful to define the \emph{k-shifted strategic hinge loss} as $\shingeloss(\beta; k) := \max\{0, 1 - y(\beta^\top x + k)\}$.

Though $\shingeloss^c$ allows for more general costs, it remains non-convex---we will return to this point in \cref{sec:min-classifier_max-cost}.
Also note that the ``correct'' transformation to use depends on the true cost. In our case this is unknown, but we show that it suffices for the dual norm to be bounded by a constant $B$:
$\forall \cinc,\ \dualnorm{\beta} \leq B$. We also let $X\in\R$ denote the maximum of $\norm{x}, \twonorm{x}$ over the training examples. We treat both $B$ and $X$ as fixed constants. Finally, we denote by $L$ the Lipschitz constant of the loss gradient, which appears in the convergence rates of the algorithms we derive. This can be generically bounded as $L \leq X + \uinv L_*$, where $L_*$ is the Lipschitz constant of the dual norm. For the common setting where $\norm{\cdot}$ is a $p$-norm, we have $L_* = \max\left(1, d^{\nicefrac{1}{2} - \nicefrac{1}{p}}\right)$. Notably, this quantity is independent of the dimension $d$ for $p \leq 2$ and scales no worse than $d^{\nicefrac{1}{2}}$ otherwise.


\paragraph{Risk and generalization.}
Denote the overall strategic hinge risk by $\shingepoprisk^c(\beta) := \E[\shingeloss^c(\beta^\top x, y)]$, with the strategic 0-1 risk defined analogously as $R_{0-1}^{\costfunc}$.
\begin{lemma}
    \label{cor-zo-shinge}
    For any cost $c\in\costset$, $R_{0-1}^{\costfunc}(\beta) \leq \shingepoprisk^c(\beta)$.
\end{lemma}
Thus, our cost-dependent strategic hinge loss is an effective proxy for the 0-1 loss.
We next establish its generalization.
\begin{theorem} 
\label{thm:generalization-main-body}
With probability $\geq 1-\delta$, for all $\beta \in\gB$ and all cost functions $\cinc$,
\begin{align*}
    R_{0-1}^{\costfunc}(\beta) &\leq \hat R_{\textrm{s-hinge}}^c(\beta) + \frac{B(4X + \uinv) + 3\sqrt{\ln \nicefrac{1}{\delta}}}{\sqrt{n}},
\end{align*}
where $\hat R$ is the empirical risk over a training set of size $n$.
\end{theorem}
This result extends (and fixes an error in) the bound for $\ell_2$-norm cost from \citet{levanon2022generalized}. The proof, found in \cref{appsec:rademacher-proof}, applies standard Rademacher bounds by decomposing the strategic hinge loss and bounding the terms separately while accounting for general strategic responses. The fact that this bound holds uniformly for \emph{all} cost functions is critical, as it allows us to apply it to the worst case cost even when that cost is unknown.

\section{The Perils of Using a Wrong Cost Function}
\label{sec:wrong-cost-difficulty}

As the agents' movement depends intricately on the precise cost function, correctly anticipating strategic response requires a good estimate of that cost. In the one-shot setting, this is further complicated by the fact that we have no access to a mechanism by which to \emph{infer} the cost (e.g., via online interaction):
we must pick a single classifier and commit to it,
without knowing the true cost function a priori.
A natural approach would be to make use of existing machinery for single-cost strategic learning, as described above, 
using some reasonable choice for the cost.
For example, one idea would be to simply pick a cost which we believe is reasonably ``close'' to the true cost, in the hope that predictive performance degrades gracefully with our error.
Certainly, this is better than blindly proceeding with the default $\ell_2$-norm. Unfortunately we find that the task of cost-robust strategic learning is much more difficult (learning theoretically) than is first apparent---it turns out that without more explicit assumptions on our guess's distance to the true cost \emph{and} on the data distribution itself, providing any sort of robustness guarantee is impossible.

\paragraph{Hardness results.}
Our first result proves that if we must to commit to a single fixed cost, unless that cost is exactly correct, minimizing the empirical risk can never provide a non-trivial data-independent guarantee:
\begin{theorem}
    \label{thm:cost-error-half-lower-bound}
    Consider the task of learning a norm-bounded binary linear classifier. Fix any two costs $c_1, c_2$ with non-equal cost matrices, and let $0 \leq \epsilon \leq \frac{1}{2}$. There exists a distribution $q$ over $\gX\times\gY$ such that:
    \begin{enumerate}
    \item For each of $c_1$ and $c_2$, there is a (different) classifier which achieves 0 error on $q$ when facing strategic response under that cost; and
    \item Any classifier which achieves 0 error on $q$ under $c_1$ suffers error $\epsilon$ under $c_2$, and any classifier which achieves 0 error on $q$ under $c_2$ suffers error $1-\epsilon$ under $c_2$.
    \end{enumerate}
\end{theorem}
\begin{figure*}[ht!]
    \centering
    \begin{subfigure}{.48\linewidth}
        \includegraphics[width=\linewidth]{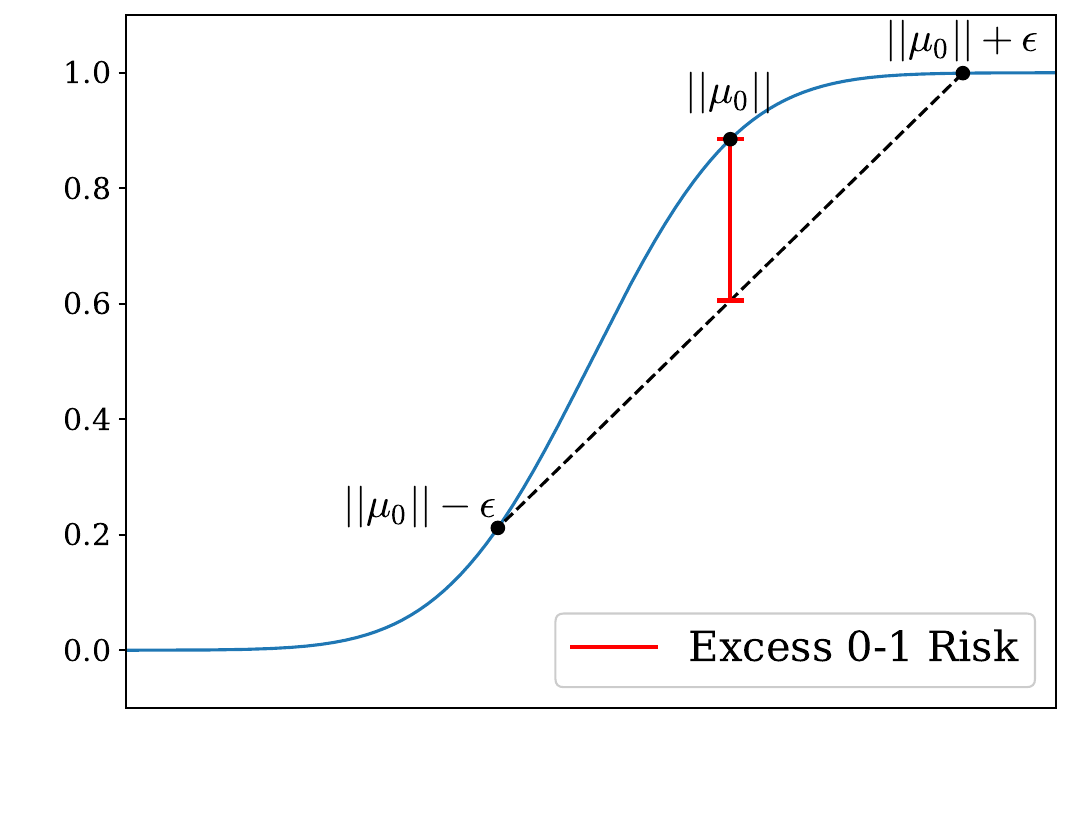}
    \end{subfigure}
    \begin{subfigure}{.495\linewidth}
        \includegraphics[width=\linewidth]{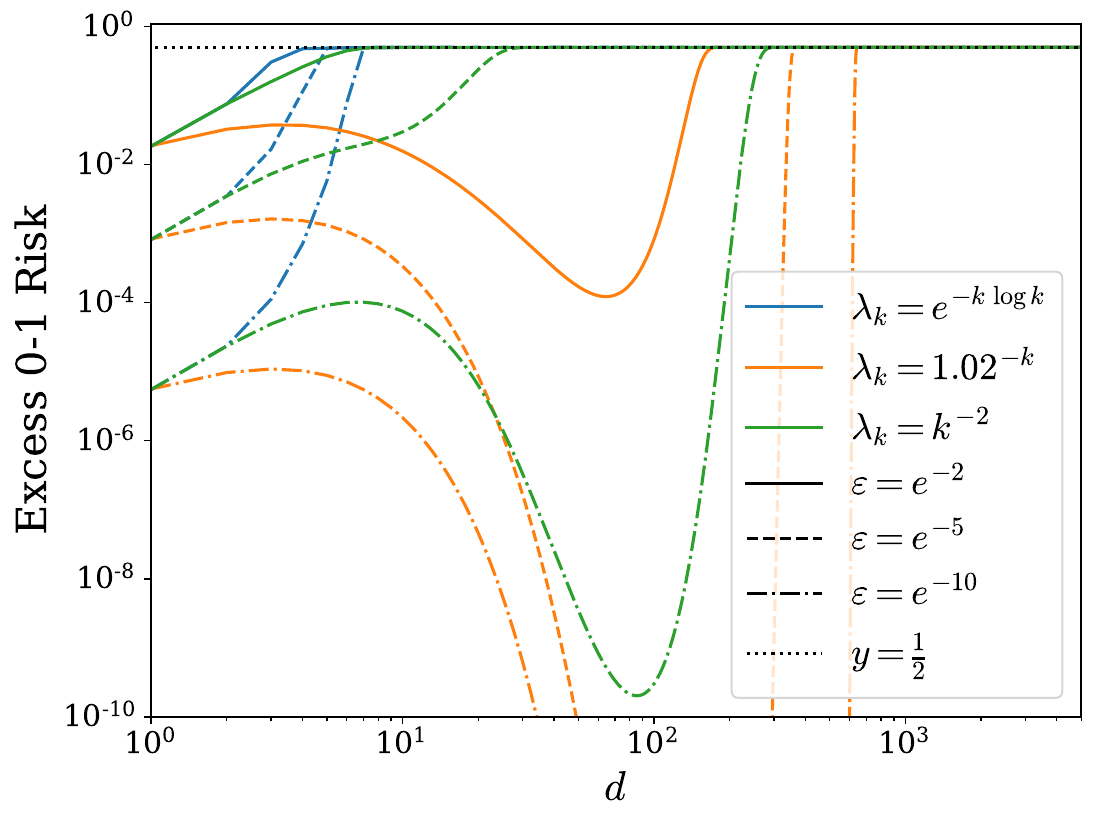}
    \end{subfigure}
    \caption{\textbf{Visualizing the excess 0-1 risk from \cref{thm:cost-error-gauss-lower-bound}. Left:} A toy illustration of the excess risk as a function of $\norm{\mu_0}$ and $\epsilon$ and where they lie on the Gaussian CDF. \textbf{Right:} Excess risk curves as a function of the dimension $d$, where $\mu_0 = \nicefrac{1}{\sqrt{d}}\; \mathbf{1}$ and $\sigma^2 = \nicefrac{1}{d}$. Color indicates the function which generates the spectrum of $\costmatrix$, where $\lambda_k$ is the $k$-largest eigenvalue in the series. Line style indicates the error $\varepsilon$ in estimating the eigenvalue in each dimension---precisely, we define $\hat\costmatrix = (1 - \varepsilon) \costmatrix$. Even with all eigenvalues estimated to error $1-e^{-10}$, excess risk grows rapidly with $d$ towards $\nicefrac{1}{2}$.}
    \label{fig:excess-risk}
\end{figure*}

\cref{thm:cost-error-half-lower-bound} says that if there is \emph{any error at all} in estimating $\costmatrix$, then any classifier which achieves perfect accuracy under our assumed cost might do no better than random (or even worse than random) when deployed. The proof (see \cref{appsec:hardness-proofs}) constructs a worst-case distribution that is chosen adversarially with respect to the error in cost estimation. This construction is quite robust in that there exists an entire space of such solutions and the lower bound decays smoothly for distributions which are close to the one we define (especially for larger error in the estimate of $\costmatrix$). However, there is hope that perhaps for non-adversarial distributions, estimating a fixed cost may be reasonable. To investigate this possibility, we also study the more natural setting of isotropic Gaussian $q(x \mid y)$. We find once again that, even in the limit of infinite data, guessing the wrong cost can be quite harmful:
\begin{theorem}
    \label{thm:cost-error-gauss-lower-bound}
    Define the distribution $q$ over $\gX\times\gY$ as $q(y=1) = q(y=-1) = \nicefrac{1}{2}$, $q(x \mid y) \overset{d}{=} \calN(y\cdot\mu_0, \sigma^2 I)$.
    Denote by $\Phi$ the standard Normal CDF.
    Let the true cost be defined as $\psdnorm{x - x'}{\costmatrix}$ with unknown cost matrix $\costmatrix$, and let $\beta^*$ be the classifier which minimizes the strategic 0-1 risk under this cost.
    
    Suppose one instead learns a classifier $\hat\beta$ by assuming an incorrect cost $\hat\costmatrix$ and minimizing the population strategic 0-1 risk under that cost: $\hat\beta := \argmin_\beta \E_q[\zoloss^{\hat c}(\beta)]$.
     Then the excess 0-1 risk suffered by $\hat\beta$ is
    \begin{align*}
        \Phi\left( \frac{\norm{\mu_0}}{\sigma} \right) - \frac{1}{2} \left( \Phi\left( \frac{\norm{\mu_0} - \epsilon}{\sigma} \right) + \Phi\left( \frac{\norm{\mu_0} + \epsilon}{\sigma} \right)\right),
    \end{align*}
    where $\epsilon := \frac{\uinv |\dualpsdnorm{\mu_0}{\hat\costmatrix} - \dualpsdnorm{\mu_0}{\costmatrix}|}{\norm{\mu_0}}$.
\end{theorem}
\cref{thm:cost-error-gauss-lower-bound} demonstrates that even for \emph{much} more benign distributions, choosing a classifier based on a slightly incorrect cost can be very non-robust. Roughly, we should expect $\epsilon = \Theta(\uinv \textrm{tr}(\hat\costmatrix^{-1} - \costmatrix^{-1}))$, scaling as our error grows along directions where $\mu_0$ is large. Crucially, this is with respect to the \emph{inverse} cost matrix $\costmatrix^{-1}$, which means that a tiny estimation error in the lower end of the eigen-spectrum can have a large effect. For example, imagine that the classes are sufficiently separated so that the optimal classifier \emph{absent} strategic behavior can get accuracy close to 1 (which is trivial in high dimensions). Then under strategic behavior, the error in guessing the smallest eigenvalue of $\costmatrix$ need only be $\Omega(\norm{\mu_0} \lambda_{\min}(\costmatrix) \lambda_{\min}(\hat\costmatrix))$ to induce excess risk of approximately $\nicefrac{1}{4}$, and it will rapidly approach $\nicefrac{1}{2}$ as $d$ grows. To give a bit more intuition for these factors, \cref{fig:excess-risk} visualizes a few simple examples of how the excess risk behaves as dimension increases, depending on the distribution of eigenvalues of $\costmatrix$ and our error in estimating them. More abstractly, we can see that if we only slightly err in estimating the cost of movement in a direction, it could cause a very large error in estimating how much agents will move in that direction.
\looseness=-1

\paragraph{Intentionally conservative guesses can still fail.} It is not immediately obvious why guessing a single cost should always have poor worst-case performance. For example, it might appear that we could just select a cost function which \emph{under}estimates the cost to move in any given direction, and therefore overestimates the movement of any given point. In being intentionally conservative, this choice seems like it should be reasonably robust to misspecification, even if not optimally so. However, this ignores the key fact that in strategic classification there is also the potential for \emph{beneficial} strategic response: a point with true label $y=1$ but incorrect prediction $\gYhat=-1$ has the potential to correct this error---to the learner's benefit---by changing its features.

The above results show that it is not enough in strategic classification to be aware of the fact that users will move in response to the classifier, nor to have an educated guess for how they will move: it is essential to account for \emph{uncertainty} in how they will move by anticipating the ways in which the true, unknown cost function may differ from what we expect. Our lower bounds thus motivate robustly learning a classifier based not on a single cost, but on a \emph{set} in which the true cost is expected to lie---this reduces potential misspecification of the response model \citep{lin2023plug} while remaining applicable in the one-shot setting.

\section{Maximizing Risk for a Fixed Classifier and Minimizing Risk for a Fixed Cost}

To identify a robust classifier when the true cost is unknown, we will optimize the worst-case strategic risk with respect to the learner-specified uncertainty set $\costset$.
Our objective is therefore to solve the min-max problem
\begin{align}
    \label{eq:minmax-objective}
    \min_{\beta} \max_{\cinc} \shingepoprisk^c(\beta).
\end{align}
As we noted in the introduction, if we consider the cost function to be inducing a classifier-specific distribution, then \cref{eq:minmax-objective} becomes an instance of \emph{distributionally robust optimization}. A typical approach to this type of problem would be to use a gradient-based method which converges to a Nash equilibrium between the learner and an adversary (who ``chooses'' the cost). To apply this in our setting, we must recast \cref{eq:minmax-objective} as two separate objectives: one for the learner minimizing $\beta$ and the other for the adversary maximizing $\costfunc$. However, the idiosyncrasies of the strategic learning setting give rise to unique challenges which must be addressed for this approach to be successful.

\subsection{Solving the Inner Max for a Fixed Classifier}
The first step to solving \cref{eq:minmax-objective} is determining a way to solve the inner maximization problem, which will serve as an important subroutine for our eventual algorithm which solves the full min-max objective. However, the max objective is non-concave. We thus begin with an efficient algorithm to address this.
\cref{alg:maxlosscost} gives pseudocode\footnote{The code calls for identifying $\dualnorm{\beta}^{\min}$ and $\dualnorm{\beta}^{\max}$. The specific way we parameterize $\costset$ makes this simple (see \cref{sec:C-reparam}).} for \textsc{MaxLossCost}
which, for a given classifier $\beta$, maximizes the strategic hinge loss over costs $\costfunc \in \costset$.
Due to \cref{eq:cost-dep_stratgic_hinge}, this amounts to finding the maximizing $k$-shift for $k=\uinv \dualnorm{\beta}$ where $\dualnorm{\beta}$ is achievable by some $\cinc$.
We next prove correctness and runtime: 

\setlength{\textfloatsep}{18pt}
\begin{algorithm}[t]
\caption{Pseudocode for \textsc{MaxLossCost}}
\label{alg:maxlosscost}
\begin{algorithmic}
\STATE \textbf{input:} Dataset $\gD = \{(x_i, y_i)\}_{i=1}^n$, Classifier $\beta$, Cost uncertainty set $\costset$, Upper bound $\uinv$.
\STATE \textbf{define:} $\dualnorm{\beta}^{\min} := \displaystyle \min_{\costmatrix \in \costset}  \dualnorm{\beta},\ \dualnorm{\beta}^{\max} := \displaystyle \max_{\costmatrix \in \costset} \dualnorm{\beta}$
\STATE \textbf{initialize:} $k \gets \dualnorm{\beta}^{\min}$,\ $r \gets \shingeemprisk(\beta; \uinv k)$
\begin{ALC@g}
    \STATE Sort training points in increasing order by $v_i := y_i (1 - y_i \beta^\top x_i)$.
    \STATE Set $j$ as first index where $v_j - \uinv k > 0$.
    \FOR{entry $v_i$ in sorted list, starting from $j$}
    \STATE Set $k = \nicefrac{v_i}{\uinv}$.
    \STATE If $k > \dualnorm{\beta}^{\max}$, break.
    \STATE Update new risk $r = \shingeemprisk(\beta; \uinv k)$. Maintain maximum $r_{\max}$ and $k_{\max}$ which induced it.
    \ENDFOR
    \STATE If $\shingeemprisk(\beta; \uinv \dualnorm{\beta}^{\max})\! >\! r_{\max}$, return $\displaystyle \argmax_{\costmatrix \in \costset} \dualnorm{\beta}$
    \STATE Otherwise, return $\costmatrix\in\costset$ such that $\dualpsdnorm{\beta}{\costmatrix} = k_{\max}$.
\end{ALC@g}
\end{algorithmic}
\end{algorithm}

\begin{lemma}
    \label{lemma:solve-max}
    For any classifier $\beta$, dataset $(\gX\times\gY)^{n}$, and uncertainty set $\costset$,
    \textsc{MaxLossCost} runs in $\gO(nd + n\ln n)$ time and returns a cost function $\cinc$ which maximizes the cost-dependent strategic hinge risk $\shingeemprisk^c(\beta)$.
\end{lemma}
\textsc{MaxLossCost}
takes advantage of the fact that the loss is piecewise linear in the \emph{scalar value $\dualnorm{\beta}$}, carefully constructing a list of the training samples which are sorted according to a particular intermediate value and traversing this list while updating the risk at the boundary of each linear component.
For a full description of the algorithm see \cref{appsec:subgradient-proof}.

\paragraph{Bounding the adversarial risk of any classifier.}
Independent of our main objective, \textsc{MaxLossCost} already allows us to derive adversarial strategic risk bounds for any classifier:
simply evaluate its worst-case empirical strategic hinge loss and then apply \cref{thm:generalization-main-body}. In principle, this means we could \emph{try} minimizing the strategic hinge loss for a single cost and we would technically be able to derive an error bound for the solution. However, this approach is not ideal for two reasons: first, as noted earlier, the strategic hinge loss is non-convex in $\beta$, so direct optimization may be unsuccessful. Second, as shown in \cref{sec:wrong-cost-difficulty} there is little reason to believe that optimizing an objective which does not account for the min-max structure will achieve good adversarial risk, making the generalization bound correct but usually unhelpful (e.g., trivial). Still, it is useful to be able to give a valid bound on the adversarial risk of whatever classifier we may hope to evaluate.

\subsection{Solving the Outer Min for a Fixed Cost}
\label{sec:min-classifier_max-cost}

We now switch to the task of finding a classifier which minimizes worst-case risk.
Note \cref{eq:minmax-objective} poses two key challenges:
(i) the strategic hinge loss is non-convex,
and (ii) optimization is further complicated by the inner max operator.
Hence, we begin with the case of a single fixed cost,
which will serve us as an intermediary step towards optimizing over
the set of all costs.
Even for a single cost, the non-convexity of the strategic hinge means that we would only be able to guarantee convergence to a stationary point.
We address this via regularization.

\paragraph{Convexification via regularization.}
Regularization is a standard means to introduce (strong) convexity:
for the hinge loss, applying (squared) $\ell_2$ regularization makes the objective strongly convex, improving optimization while simultaneously preventing overfitting.
Since we would want to regularize our objective anyways, one possible solution would be to add just enough $\ell_2$ regularization to convexify it.
While sound in principle, this is inappropriate for our setting because it does not account for the cost-specific nature of strategic response, so the amount of regularization needed would be extreme. Specifically, we prove a lower bound on the required $\ell_2$ regularization to ensure convexity of the strategic hinge loss:
\begin{theorem}
    \label{thm:cvx-regularization-ell2}
    Let $\norm{\cdot}$ be a $p$-norm and fix a cost matrix $\costmatrix$. There is a distribution $q$ such that the $\ell_2$-regularized loss $\shingepoprisk^c(\beta) + \lambda \uinv \twonorm{\beta}$ is non-convex unless $\lambda \geq \twonorm{\costmatrix^{-1}}$.
\end{theorem}
Thus, the necessary $\ell_2$ regularization would be cost-dependent and quite large, scaling with the inverse of the smallest eigenvalue.
Moreover, since we want to optimize this risk over all possible costs, this means that it would need to scale with the largest spectral norm among \emph{all} inverse cost matrices in $\costset$. 

\paragraph{Dual norm regularization.}
Luckily, we observe that for a given cost $\costfunc$ we can instead apply a small amount of \emph{dual norm} regularization to $\beta$. This naturally accounts for the problem's structure and eliminates the need for excessive penalization.
Our dual-regularized objective is
    $\shingepoprisk^c(\beta) + \lambda \uinv \dualnorm{\beta}$,
where $\lambda$ determines the regularization strength.
Since the dual norm $\dualnorm{\beta}$ is implicitly defined via the cost matrix,
the following is straightforward: 
\begin{proposition}
    \label{thm:cvx-regularization1}
    For any norm $\norm{\cdot}$, cost matrix $\costmatrix$, and distribution $q$,
    the dual-regularized loss $\shingepoprisk^c(\beta) + \lambda \uinv \dualnorm{\beta}$ is guaranteed to be convex for all $\lambda \geq q(y=1)$.
\end{proposition}
In fact, we can show that $q(y=1) \uinv \dualnorm{\beta}$ is the \emph{minimum} regularization that ensures convexity: for any $\lambda$, if the objective in \cref{thm:cvx-regularization-ell2} is convex, then $q(y=1) \dualnorm{\beta} \leq \lambda \twonorm{\beta}$. Proofs are in \cref{appsec:cvx-reg-proof}.
\cref{thm:cvx-regularization1} shows that a small, cost-independent amount of dual norm regularization ensures convexity of the outer min problem for \emph{any} cost function, making it the natural choice in this setting.

\section{Efficiently Identifying the Minimax-Optimal Classifier}

We have shown how to efficiently solve for the maximizing cost $\cinc$ for any classifier $\beta$, allowing us to apply the corrected generalization bound in \cref{thm:generalization-main-body} and get an upper bound on the adversarial 0-1 strategic risk. We have also seen that with dual norm regularization, optimizing the strategic hinge loss becomes tractable and also accounts for the unique structure of strategic response.
The remaining challenge is to combine these two methods to find the overall solution to the min-max objective.

\paragraph{Solving a different loss for each cost simultaneously.}
An interesting consequence of using the dual norm is that for any fixed cost, the ``correct'' regularization is a function of that cost. Since we want to minimize this objective with respect to the entirety of $\costset$, we absorb a \emph{separate} dual norm regularization term into the min-max objective for each possible cost, defining our new objective as
\begin{align}
    \label{eq:regularized-minmax-objective}
    \min_{\beta} \max_{\cinc} \bigl[ \shingepoprisk^c(\beta) + \lambda \uinv \dualpsdnorm{\beta}{\costmatrix} \bigr].
\end{align}
Contrast this with typical regularization, which (e.g. for $\ell_2$) would instead attempt to solve $\min_{\beta} \bigl[  \max_{\cinc}  \bigl[ \shingepoprisk^c(\beta) \bigr] + \lambda \twonorm{\beta} \bigr]$,
where regularization does not depend on the inner maximization.
For brevity, in the remainder of this work we leave the regularization implicit, writing simply $\shingepoprisk^c(\beta)$. \cref{thm:cvx-regularization1} shows that a single choice of $\lambda$ always suffices: whereas the structure-aware regularizer depends on the cost, the ideal coefficient does not. By setting $\lambda$ appropriately and optimizing this new objective we account for the cost uncertainty ``for free'', effectively solving the problem across all costs in $\costset$ \emph{simultaneously} with a separate, appropriate regularization term for each. Importantly, this means our solution will be optimal with respect to the regularized objective. But once a solution is found, we can evaluate the adversarial risk without the regularization term and \cref{thm:generalization-main-body} will still apply.

\subsection{Solving the Objective in the Full-Batch Setting with the Subgradient Method}
\label{sec:C-reparam}
Given a train set $\{(x_i, y_i)\}_{i=1}^n$, perhaps the simplest idea for solving \cref{eq:regularized-minmax-objective} is to follow the gradient of the empirical adversarial risk. The validity of this approach is not trivial because of the min-max formulation, but we show that it is indeed correct by invoking the following result:
\begin{theorem}[Danskin's Theorem \citep{bertsekas1997nonlinear}]
    Suppose $f : \R^n \times \mathcal{Z} \to \R$ is a continuous function, where $\mathcal{Z} \subset \R^m$ is a compact set. Define $g(x) := \max_{z\in\mathcal{Z}} f(x, z)$. Then $g(x)$ is convex in $x$ if $f(x, z)$ is convex in $x$ for every $z\in\mathcal{Z}$. Furthermore, $\partial_x g(x) = \textrm{\emph{Conv}}\left\{\partial_x f(x, z)\;:\; z\in\argmax_z f(x, z) \right\}$, where $\partial$ is the subdifferential and ``\textrm{\emph{Conv}}'' indicates the convex hull. 
\end{theorem}
Note that this requires $\mathcal{Z}\subset \R^m$, whereas we introduced $\costset$ as a set of PD matrices in $\R^{d\times d}$.
Using the assumption that $\costset$ is compact and convex,
we resolve this by associating to each cost matrix $\costmatrix$ a vector of $d$ eigenvalues $\sigma_1^2\ldots \sigma_d^2$ under a fixed basis, with each eigenvalue $\sigma_i^2$ constrained to lie in the set $[\mincostscale{i}^2, \maxcostscale{i}^2]$. These lower and upper bounds allow us to re-parameterize the cost uncertainty set so that it is now a subset of $\R^d$. One remaining technicality is that this requires a fixed basis which may not be shared with the true cost. However, as a consequence of \cref{lemma:dualnorm-soln}, our results will still be meaningful so long as there exists a $\cinc$ which induces the same dual norm as the true cost. Without loss of generality, we suppose this basis is the identity.

Thus, we let $\mathcal{Z}$ denote the reparameterized uncertainty set $\costset$ and plug in $\shingeemprisk^{c=z}$ for $f$. We conclude that the subderivative of the worst-case strategic hinge loss is given by the subderivative of the loss evaluated at any cost in $\costset$ which maximizes it.
We can optimize the objective via the subgradient method (\cref{alg:subgradient} in the Appendix), where each subgradient evaluation involves a call to the maximization subroutine \textsc{MaxLossCost}.
We then combine this with well-known convergence results for the subgradient method and bound the generalization error via \cref{thm:generalization-main-body}, giving the complete result:
\begin{theorem}
    \label{thm:subgradient}
    Suppose we run \cref{alg:subgradient} for $T$ iterations and get classifier $\hat\beta$. With probability $\geq 1-\delta$, the worst-case 0-1 strategic loss over costs in $\costset$ can be bounded by
    \begin{align*}
    \max_{\cinc} R_{0-1}^{\costfunc}(\hat\beta) \leq &
     \min_{\beta} \max_{\cinc} \shingepoprisk^{\costfunc}(\beta) \\
     & + \mathcal{O}\left(\frac{LB}{\sqrt{T}} + (X+\uinv) \frac{B + \sqrt{\ln \nicefrac{1}{\delta}}}{\sqrt{n}} \right).
    \end{align*}
\end{theorem}
The proof appears in \cref{appsec:subgradient-proof}. Despite the difficulty posed by strategic response \emph{and} an unknown cost (as made apparent in \cref{sec:wrong-cost-difficulty}), \cref{thm:subgradient} shows that robust one-shot learning is possible with sufficient consideration of the underlying structure. 

\subsection{Solving the Objective in the Minibatch Setting with Stochastic Mirror Descent-Ascent}
Occasionally the number of training points $n$ may be sufficiently large that gradient descent is intractable, or the full dataset may not be available all at once. We would still like to be able to solve \cref{eq:minmax-objective}, even when we cannot evaluate the full gradient. Unfortunately, the stochastic subgradient method is not an option here: evaluating the subderivative requires that we identify the cost that maximizes the \emph{population} objective, which cannot be done on the basis of a subsample. As an alternative, we turn to a method from convex-concave optimization known as \emph{Stochastic Mirror Descent-Ascent} (SMDA) \citep{nemirovski2009robust} which iteratively optimizes the min and max players to converge to a Nash equilibrium.

\begin{algorithm}[t]
\caption{Stochastic Mirror Descent-Ascent on the regularized strategic hinge loss}
\label{alg:mirror-descent}
\begin{algorithmic}
\STATE \textbf{Input:} Batch size $n$, Iterations $T$, Costs $\costset$, Upper bound $\uinv$, Regularization $\lambda$, Discretization $\epsilon$, Step sizes $\eta_q, \eta_\beta$.
\STATE \textbf{define:}\; $\beta^{(0)} \gets \mathbf{0}$
\STATE \qquad\quad\ \ $v \gets [\epsilon(\mincostscale{1}^{-2}-\maxcostscale{1}^{-2}),\ldots, \epsilon(\mincostscale{d}^{-2}-\maxcostscale{d}^{-2})]^\top$
\STATE \qquad\quad\ \ $q^{(0)} \gets \epsilon \mathbf{1}$
\FOR{$t = 1,\ldots,T$}
\STATE Draw samples $\{(x_i, y_i)\}_{i=1}^n$.
\STATE $c_{k} \gets \vec{\sigma_u}^{-2} + k\cdot v, \quad k\in\{1,\ldots,\lceil\nicefrac{1}{\epsilon}\rceil\}$
\STATE $q' \gets q^{(t-1)}$
\STATE $q'_{k} \gets q'_{k} \exp(\eta_q \shingeemprisk^{c_k}(\beta^{(t-1)}))$
\STATE $q^{(t)} \gets q' / \sum_k q'_{k}$
\STATE $\beta^{(t)} \gets \beta^{(t-1)} - \eta_{\beta} \left( \sum_k q_{k}^{(t)} \nabla \shingeemprisk^{c_k}(\beta^{(t-1)}) \right)$
\ENDFOR
\STATE \textbf{return} $\hat\beta := \frac{1}{T} \sum_{t=1}^T \beta^{(t)}$
\end{algorithmic}
\end{algorithm}

\paragraph{Modifying mirror ascent to apply to our setting.} Since the objective is convex in $\beta$, minimization is straightforward. However, recall that the maximization over costs is non-concave. Previously we got around this by solving for the maximizing cost in a non-differentiable manner and then invoking Danskin's theorem---but the guarantees given by \citet{nemirovski2009robust} require that we take iterative steps of gradient ascent on the adversary's (assumed concave) objective. We address this via an \emph{algorithmic $\epsilon$-net}. Specifically, we can relax \cref{eq:regularized-minmax-objective} to a maximization over a \emph{finite} set of costs $\mathcal{S}\subset \costset$ such that the solution to this new objective is provably close to that of the original problem. Given such a set, the new objective becomes $\min_{\beta} \max_{\costfunc\in \mathcal{S}} \shingepoprisk^c(\beta)$.
Observe that the solution to this objective is equivalent to the solution to
\begin{align}
    \label{eq:convex-comb-obj}
    \min_{\beta} \max_{q\in\Delta(|\mathcal{S}|)} \sum_{c_i\in \mathcal{S}} q_i\shingepoprisk^{c_i}(\beta),
\end{align}
where $\Delta(|\mathcal{S}|)$ is the $|\mathcal{S}|$-simplex. In other words, we can solve this problem over all convex combinations of the risks under the different costs in $\mathcal{S}$ and arrive at the same solution \citep{rosenfeld2022domain, rosenfeld2022online}. Crucially, unlike our original objective, \cref{eq:convex-comb-obj} is convex-concave, so it can be efficiently optimized via SMDA! The correctness of this relaxation relies on the property proven in \cref{lemma:dualnorm-soln}: strictly speaking, the strategic loss depends only the \emph{scalar} dual norm, not on the full cost matrix.

The gap between the solution to the original objective and \cref{eq:convex-comb-obj} scales inversely with the fineness of discretization, so we want the maximum distance between any cost $\cinc$ and its closest neighbor in $\mathcal{S}$ to be as small as possible. However, the memory and compute requirements of SMDA scale linearly with $|\mathcal{S}|$---and error scales as $\sqrt{\ln |\mathcal{S}|}$---so we also cannot let it grow too large. To balance these two considerations, we carefully construct a discretization over the ``diagonal'' of the space of eigenvalue intervals in $\costset$, leading to a set $\mathcal{S}$ with cardinality $\Theta\left(\frac{T D}{\ln T} \right)$, where $T$ is the number of iterations and $D := \max_i \nicefrac{1}{\mincostscale{i}^2} - \nicefrac{1}{\maxcostscale{i}^2}$ is the diameter of $\costset$. The exact construction for $\gS$ appears in \cref{alg:mirror-descent}. With this careful discretization, the computational cost to achieve a fixed sub-optimality gap is \emph{dimension independent} for $p$-norms with $p \leq 2$, and the worst-case scaling (when $p=\infty$) is $\mathcal{O}\left(\nicefrac{d^2}{\ln d}\right)$. In contrast, a typical $\epsilon$-net would have $|\gS| = \Theta(\exp(d))$, with memory and compute scaling commensurately.

\begin{theorem}
    Suppose we run \cref{alg:mirror-descent} for $T$ rounds with discretization $\epsilon = \Theta\left(\frac{\ln T}{T D}\right)$ and get classifier $\hat\beta$.
    Then over the randomness of the stochastic gradients, its expected sub-optimality\footnote{Exponential concentration of the adversarial risk to its expectation follows from \citet{nemirovski2009robust}, Proposition 3.2.} is bounded by
    \begin{align*}
        \E[\max_{\cinc} \shingepoprisk^{\costfunc}(\hat\beta)] &\lesssim \min_{\beta} \max_{\cinc} \shingepoprisk^{\costfunc}(\beta) \\
        &+\mathcal{O}\left(\frac{LB}{\sqrt{T}}  + B (X+\uinv) \sqrt{\frac{\ln TD}{T}} \right).
    \end{align*}
\end{theorem}
The proof can be found in \cref{appsec:mirror-proofs}. Remarkably, convergence to the population minimax solution occurs at the same rate as the full-batch setting up to logarithmic factors---the cost of stochasticity is surprisingly small. Once again, we can use \textsc{MaxLossCost} to evaluate the adversarial strategic risk and then apply \cref{thm:generalization-main-body} to get a generalization bound on the worst-case 0-1 error.

\section{Conclusion}
This paper studies robust strategic learning under unknown user costs in the challenging one-shot setting.
Our results suggest that uncertainty in how users respond should be considered an integral aspect of strategic learning:
motivated by realistic problem domains which  permit only a single action (which must be immediately effective),
we provide a learning framework based on distributionally robust optimization for modeling---and robustly handling---this uncertainty. We begin by showing that even a miniscule error in estimating the true cost can cause substantial error in deployment, motivating the use of an uncertainty set of possible costs.
Next, we propose a natural proxy to the intractable min-max objective over this set, and we design two efficient algorithms for different settings that converge to the empirical solution, for which we also provide generalization guarantees.

Our approach highlights the value of dual norm regularization,
which ensures good performance while accounting for the structure of strategic behavior coupled with cost uncertainty. Such structure can both
harm accuracy (on negative examples) and help it (on positive examples).
One important implication is that it often does not suffice to simply be more conservative, as this can fail to take advantage of settings where the learner's and agents' incentives are aligned. Rather, maintaining robustness without substantially sacrificing accuracy requires more careful consideration of and accounting for the interests and actions of future users.\looseness=-1

\section*{Impact Statement}
This work studies the problem of strategic classification under unknown user costs.
While our focus is entirely theoretical,
it is important to acknowledge that,
in a realistic setting, 
the learning task which we consider would have human beings as the targets for prediction.
This is in fact our primary motivation:
correctly accounting for strategic behavior in learning can benefit not only the system, but also the general population of users \citep{milli2019social,levanon2021strategic,levanon2022generalized}.
Questions of how to act under uncertainty or partial information are immanent in strategic learning;
our work offers but one way of accounting for system uncertainty regarding user behavior, with a particular focus on uncertainty stemming from unknown user costs.
Other sources of uncertainty therefore lie outside our scope,
as does human behavior which deviates from our assumed rational response model (Eq.~\eqref{eq:best_response}),
as well as considerations beyond minimizing worst case risk.
Although our analysis applies to a relatively large and flexible class of cost functions (broader than what is typically considered in the literature),
a clear limitation of our work is that it is unclear if it applies to costs outside this class.
A more subtle but nonetheless important limitation is that while our approach provides strong worst-case guarantees for any cost in the uncertainty set $\costset$, it does not state what will happen if the `true' cost $\costfunc^*$ lies outside of the chosen $\costset$.
We conjecture that under appropriate smoothness assumptions, worst-case guarantees should degrade gracefully as $\costfunc^*$ moves away from $\costset$,
but leave this as an open question for future work.
Finally, although the algorithms we propose are theoretical,
they may nonetheless be applicable in certain settings.
Given this, we implore any future use of our algorithms to be carried out in a transparent and accountable manner---as should be expected from any tool developed under the general framework of strategic learning.
\looseness=-1

\section*{Acknowledgements}
Thanks to Roni Rosenfeld for helpful discussions in developing the motivation for this work. This work is supported by the Israel Science Foundation grant no. 278/22.

\clearpage
\bibliography{arxiv}
\bibliographystyle{plainnat}

\clearpage
\onecolumn

\appendix
\section{Proofs of Hardness Results}
\label{appsec:hardness-proofs}
We restate the theorems for convenience.

\subsection{Proof of \cref{thm:cost-error-half-lower-bound}}

\begin{theorem}
Consider the task of learning a norm-bounded linear classifier. Fix any two costs $c_1, c_2$ with non-equal PD matrices, and let $0 \leq \epsilon \leq \frac{1}{2}$. There exists a distribution $p$ over $\gX\times\gY$ such that:
\begin{enumerate}
    \item For each of $c_1$ and $c_2$, there is a (different) classifier which achieves 0 error on $p$ when facing strategic response under that cost; and
    \item Any classifier which achieves 0 error on $p$ under cost $c_1$ suffers error $\epsilon$ under cost $c_2$, and any classifier which achieves 0 error on $p$ under cost $c_2$ suffers error $1-\epsilon$ under cost $c_2$.
\end{enumerate}
\end{theorem}

\begin{proof}
    We will construct two distributions, one over the conditional $q(x \mid y=1)$ and the other over $q(x\mid y=-1)$, and combine them via the mixture $q(y=1) = \epsilon$, $q(y=-1) = 1-\epsilon$. Let $\costmatrix_1, \costmatrix_2$ denote the cost matrices for $c_1, c_2$ respectively. Let $B$ denote the upper bound on the classifier norm. Pick any $\beta^*$ such that $\norm{\beta^*}_{\costmatrix_1} \neq \norm{\beta^*}_{\costmatrix_2}$, and define $r := \frac{\uinv \norm{\beta^*}}{3B} (\norm{\beta^*}_{\costmatrix_1} - \norm{\beta^*}_{\costmatrix_2})$. WLOG, suppose $r > 0$. Focusing on the negatively labeled portion, consider the $(d-1)$-dimensional plane in $\gX$ defined by $\beta^{*\top} x = -r$. We let $q(x \mid y=-1)$ be any distribution with full support over that plane. Similarly, define the conditional distribution $q(x \mid y=1)$ as a full-support distribution over the plane defined by $\beta^{*\top} x = r$. Thus the only classifier which achieves 0 error must have $\hat\beta=\alpha\beta^*$ for some scaling factor $\nicefrac{B}{\norm{\beta^*}} \geq \alpha > 0$. The only remaining degree of freedom is the bias term $\hat\beta_0$.

    Consider the case where the true cost is $c_1$. If we wish to correctly classify the negative points under strategic response, we must classify the negatively labeled plane with margin greater than $\uinv \norm{\beta^*}_{\costmatrix_1}$. However, if we wish to do the same with the positive points, the margin for that plane must be less than or equal to this same value. Formally, we must have
    \begin{align}
        \beta^{*\top} x = -r &\implies \hat\beta^\top x + \hat\beta_0 < -\uinv \norm{\beta^*}_{\costmatrix_1},\\
         \beta^{*\top} x = r &\implies \hat\beta^\top x + \hat\beta_0 \geq -\uinv \norm{\beta^*}_{\costmatrix_1},
    \end{align}
    which, remembering that $\hat\beta = \alpha\beta^*$, immediately implies 
    \begin{align}
    	\hat\beta_0 - \alpha r < -\uinv \norm{\beta^*}_{\costmatrix_1},\\
	\hat\beta_0 + \alpha r \geq -\uinv \norm{\beta^*}_{\costmatrix_1}.
    \end{align}
    Thus, we have
    \begin{align}
    	-\uinv \norm{\beta^*}_{\costmatrix_1} - \alpha r \leq \hat\beta_0 < -\uinv \norm{\beta^*}_{\costmatrix_1} + \alpha r,
    \end{align}
    and since $\alpha r > 0$, this describes the non-empty set of all classifiers which achieve 0 error under cost $c_1$. By an analogous argument we can construct the set of classifiers which achieve 0 error under cost $c_2$.
    
    However, observe that 
    \begin{align}
        2\alpha r &= \frac{2}{3} \frac{\norm{\beta^*}}{B} \uinv  (\norm{\beta^*}_{\costmatrix_1} - \norm{\beta^*}_{\costmatrix_2}) \\
        &<  \uinv  (\norm{\beta^*}_{\costmatrix_1} - \norm{\beta^*}_{\costmatrix_2}),
    \end{align}
    and therefore
    \begin{align}
    	-\uinv \norm{\beta^*}_{\costmatrix_1} + \alpha r &< -\uinv \norm{\beta^*}_{\costmatrix_2} - \alpha r.
    \end{align}
    This means that the upper bound for any $\hat\beta_0$ which achieves 0 error under cost $c_1$ cannot satisfy the lower bound for cost $c_2$, which means the positively labeled plane will be classified negatively with a margin that is too large for strategic response, causing error $\epsilon$. By a symmetric argument, any $\hat\beta_0$ which achieves 0 error under cost $c_2$ cannot satisfy the upper bound for cost $c_1$, implying the negatively labeled plane will strategically shift and cause error $1 - \epsilon$. Finally, we remark that while the distribution defined here is not absolutely continuous on $\gX$, this can be remedied by simply making the distribution a product of the constructed planar distribution and a uniform distribution along the direction $\beta^*$ with width sufficiently small (e.g, width $c \cdot \alpha r$ for $c \ll 1$). \qedhere
\end{proof}

\subsection{Proof of \cref{thm:cost-error-gauss-lower-bound}}
\begin{theorem}
    Define the distribution $q$ over $\gX\times\gY$ as $q(y=1) = q(y=-1) = \nicefrac{1}{2}$, $q(x \mid y) \overset{d}{=} \calN(y\cdot\mu_0, \sigma^2 I)$.
    Denote by $\Phi$ the standard Normal CDF.
    Let the true cost be defined as $\psdnorm{x - x'}{\costmatrix}$ with unknown cost matrix $\costmatrix$, and let $\beta^*$ be the classifier which minimizes the strategic 0-1 risk under this cost.\looseness=-1
    
    Suppose one instead learns a classifier $\hat\beta$ by assuming an incorrect cost $\hat\costmatrix$ and minimizing the population strategic 0-1 risk under that cost: $\hat\beta := \argmin_\beta \E_q[\zoloss^{\hat c}(\beta)]$.
     Then the excess 0-1 risk suffered by $\hat\beta$ is
    \begin{align*}
        \Phi\left( \frac{\norm{\mu_0}}{\sigma} \right) - \frac{1}{2} \left( \Phi\left( \frac{\norm{\mu_0} - \epsilon}{\sigma} \right) + \Phi\left( \frac{\norm{\mu_0} + \epsilon}{\sigma} \right)\right),
    \end{align*}
    where $\epsilon := \frac{\uinv |\dualpsdnorm{\mu_0}{\hat\costmatrix} - \dualpsdnorm{\mu_0}{\costmatrix}|}{\norm{\mu_0}}$.
\end{theorem}
\begin{proof}
    By the Neyman-Pearson Lemma, the classifier which minimizes non-strategic 0-1 risk will be the one which predicts $\sign(\ln \frac{p(y=1\mid x)}{p(y = 0 \mid x)})$, which gives $\beta^{*} = 2\mu_0$. To account for strategic response, we observe that as proven in \cref{lemma:dualnorm-soln}, each user will have $\delta(x)$ in the direction of $\beta^*$ which induces a change in their predicted value by at most $\uinv \dualnorm{\beta^*} = 2 \uinv \dualpsdnorm{\mu_0}{\costmatrix}$; since we are considering the 0-1 loss, this is equivalent to \emph{every} user shifting in this way.
    
    It is immediate that to find the corresponding minimizer of the population 0-1 \emph{strategic} risk, we can simply add a negative bias equal to the change induced by this shift, because this maintains the same labels as before on all points, \emph{after} they strategically shift. Therefore, the strategic risk minimizer will be $2 \mu_0^\top x - 2\uinv \dualpsdnorm{\mu_0}{\costmatrix}$. By the same argument, the minimizer of the 0-1 strategic risk under the incorrect cost matrix $\hat\costmatrix$ is $2 \mu_0^\top x - 2\uinv \dualpsdnorm{\mu_0}{\hat\costmatrix}$. It remains to derive a lower bound on their difference in risk.

    First note that as argued above, the 0-1 strategic risk of the minimizer for the correct cost is the same as the non-strategic risk of the non-strategic solution. For the both positively and negatively labeled points, this is equal to $\Phi\left( -\frac{\norm{\mu_0}}{\sigma} \right)$ due to rotational symmetry. To determine the risk of the incorrect solution, we can identify the regions whose label will differ under that classifier and bound the measure of those regions. Define $\gamma = 2\uinv (\dualpsdnorm{\mu_0}{\hat\costmatrix} - \dualpsdnorm{\mu_0}{\costmatrix})$ as the difference in the two solutions' predictions on all $x$. Due to the symmetry of the positive and negative conditional distributions, we can assume WLOG that $\gamma > 0$, i.e., the incorrect classifier assigns a smaller prediction to all $x$. This means it will have less risk on a region of negative points and more on positive. Specifically, the two classifiers will differ on all points for which the true strategic-optimal classifier assigns a value \emph{before strategic response} which lies in $(-2\uinv \dualpsdnorm{\mu_0}{\costmatrix}, -2\uinv \dualpsdnorm{\mu_0}{\costmatrix} + \gamma)$; those assigned a value less than $-2\uinv \dualpsdnorm{\mu_0}{\costmatrix}$ will receive a negative prediction from both classifiers, and those assigned a value greater than $-2\uinv \dualpsdnorm{\mu_0}{\costmatrix} + \gamma$ will still be close enough to the decision boundary that they can shift to achieve a positive label prediction from $\hat\beta$. Formally, this region is
    \begin{align}
        \{x\mid -2\uinv \dualpsdnorm{\mu_0}{\costmatrix} < \beta^{*\top}x < -2\uinv \dualpsdnorm{\mu_0}{\costmatrix} + \gamma\} &= \{x\mid 0 <  \mu_0^\top x  < \nicefrac{\gamma}{2} \}.
    \end{align}
    Observe this region depends only on the value $\mu_0^\top x$. Since the negative points are distributed as $\calN(-\mu_0, \sigma^2 I)$, this term has the distribution $\mu_0^\top x \sim \calN(-\normsq{\mu_0}, \sigma^2 \normsq{\mu_0})$. Therefore, the measure of this region under $q(x \mid y=-1)$ is 
    \begin{align}
        &\ \Phi\left( \frac{\nicefrac{\gamma}{2} + \normsq{\mu_0}}{\sigma \norm{\mu_0}} \right) - \Phi\left( \frac{\norm{\mu_0}}{\sigma} \right) \\
        &= \int_{0}^{\frac{\gamma}{2 \norm{\mu_0}}} \rho\left( \frac{\norm{\mu_0} + z}{\sigma} \right)\;dz,
    \end{align}
    This is the amount by which the risk of the incorrect solution will \emph{decrease} on the negative points.
    Likewise, the risk will \emph{increase} on positive points in this region, which under $q(x\mid y=1)$ has measure
    \begin{align}
        &\ \Phi\left( \frac{\nicefrac{\gamma}{2} - \normsq{\mu_0}}{\sigma \norm{\mu_0}} \right) - \Phi\left( \frac{-\norm{\mu_0}}{\sigma} \right) \\
        &= \int_{0}^{\frac{\gamma}{2 \norm{\mu_0}}} \rho\left( \frac{-\norm{\mu_0} + z}{\sigma} \right)\;dz,
    \end{align}
    Therefore, the overall increase to risk will be
    \begin{align}
        \frac{1}{2} \int_{0}^{\epsilon} \rho\left( \frac{-\norm{\mu_0} + z}{\sigma} \right) - \rho\left( \frac{\norm{\mu_0} + z}{\sigma} \right)\;dz,
    \end{align}
    where $\epsilon := \frac{\uinv |\dualpsdnorm{\mu_0}{\hat\costmatrix} - \dualpsdnorm{\mu_0}{\costmatrix}|}{\norm{\mu_0}}$. By applying the fundamental theorem of calculus and the fact that $\Phi(x) = 1 - \Phi(-x)$ we arrive at the stated equality.
\end{proof}

\section{Proof of \cref{thm:cvx-regularization-ell2,thm:cvx-regularization1}}
\label{appsec:cvx-reg-proof}
\begin{theorem}
    Let $\norm{\cdot}$ be a $p$-norm and fix a cost matrix $\costmatrix$. For any distribution $q$ on $\gX\times\gY$ with $q(y=1) =: \tau^+$, the dual-regularized loss $\shingepoprisk^c(\beta) + \lambda \uinv \dualnorm{\beta}$ is guaranteed to be convex for $\lambda \geq \tau^+$. In contrast, the $\ell_2$-regularized loss $\shingepoprisk^c(\beta) + \lambda \uinv \twonorm{\beta}$ is non-convex unless $\lambda \geq \tau^+ \twonorm{\costmatrix^{-\nicefrac{1}{2}}}$.
\end{theorem}

\newcommand{\qnorm}[1]{\norm{#1}_q}
\newcommand{\pnorm}[1]{\norm{#1}_p}
\begin{proof}
    Writing out the dual-regularized loss with the full definition of the strategic hinge,
    \begin{multline*}
        \shingepoprisk^c(\beta) + \lambda \uinv \dualnorm{\beta} = \E_{(x,y)\sim q}[\max\{0, 1-y(\beta^\top x + \uinv \dualnorm{\beta})\}] + \lambda \uinv \dualnorm{\beta} \\
        = \tau^+ \E_{q(x \mid y=1)}[\max\{0, 1- \beta^\top x - \uinv \dualnorm{\beta}\} + \nicefrac{\lambda}{\tau^+} \uinv \dualnorm{\beta}] + \\
         (1-\tau^+) \E_{q(x \mid y=-1)}[\max\{0, 1 + \beta^\top x + \uinv \dualnorm{\beta}\}].
    \end{multline*}
    The last term is already convex in $\beta$. Rewriting the first term, we get
    \begin{align}
        \tau^+ \E_{q(x\mid y=1)}[\max\{\nicefrac{\lambda}{\tau^+} \uinv \dualnorm{\beta}, 1- \beta^\top x + \uinv \left(\nicefrac{\lambda}{\tau^+} - 1\right) \dualnorm{\beta}\}]
    \end{align}
    For $\lambda \geq \tau^+$ this is the expectation of the maximum of two convex functions, and thus the full loss is convex.
    
    To see why the $\ell_2$ norm requires much stronger regularization, consider again the above term with regularization $\lambda \uinv \twonorm{\beta}$:
    \begin{align}
        \tau^+ \E_{x\mid y=1}\left[ \max\left\{ \frac{\lambda}{\tau^+} \uinv \twonorm{\beta}, 1 - \beta^\top x + \uinv \left(\frac{\lambda}{\tau^+} \twonorm{\beta} - \dualnorm{\beta}\right) \right\} \right].
    \end{align}
    The first term in the max is convex---so to show non-convexity, we will consider values where the second term is larger. Recalling that $\dualnorm{v} := \qnorm{\costmatrix^{-1/2} v}$, write the second term as a function $f(\beta) := 1 - \beta^\top x + \uinv \left(\frac{\lambda}{\tau^+} \twonorm{\beta} - \qnorm{\costmatrix^{-1/2} \beta}\right)$, and thus
    \begin{align}
        \nabla f(\beta) &= -x + \uinv \left(\frac{\lambda}{\tau^+} \frac{\beta}{\twonorm{\beta}} - \costmatrix^{-1/2} \left( \frac{\partial}{\partial v} \qnorm{v}\big|_{v = \costmatrix^{-1/2} \beta_2} \right) \right).
    \end{align}
    Recall that a function is convex if and only if for all $x, y$ in its domain,
    \begin{align}
        f(x) - f(y) &\geq \nabla f(y)^\top (x-y).
    \end{align}
    This means that $f$ is convex only if for \emph{all} vectors $\beta_1, \beta_2 \in \mathcal{B}$,
    \begin{align}
        \label{eq:beta-convexity}
        \frac{\lambda}{\tau^+} ( \twonorm{\beta_1} - \twonorm{\beta_2} ) - \qnorm{\costmatrix^{-1/2} \beta_1} + \qnorm{\costmatrix^{-1/2} \beta_2} &\geq \left(\frac{\lambda}{\tau^+} \frac{\beta_2}{\twonorm{\beta_2}} - \costmatrix^{-1/2} \left( \frac{\partial}{\partial v} \qnorm{v}\big|_{v = \costmatrix^{-1/2} \beta_2} \right) \right)^\top (\beta_1 - \beta_2).
    \end{align}
    Without loss of generality, suppose $\costmatrix$ is diagonal. Let $v_i,\ i\in[d]$ denote the eigenvectors of $\costmatrix$ with \emph{decreasing} eigenvalues $\sigma_1^2, \sigma_2^2,\ldots$. Choose any $\beta_2 = \sum_i \lambda_i v_i$ with non-negative $\lambda_i$ and $\lambda_d = 0$. Thus $\costmatrix^{-1/2} \beta_2 = \sum_i \left( \nicefrac{\lambda_i}{\sigma_i} \right) v_i$. Then 
    \begin{align}
        \beta_2^\top \costmatrix^{-1/2} \frac{\partial}{\partial v} \qnorm{v}\big|_{v = \costmatrix^{-1/2} \beta_2} &= \sum_i \left(\nicefrac{\lambda_i}{\sigma_i}\right) \cdot \left(\frac{\left| \nicefrac{\lambda_i}{\sigma_i} \right|}{\qnorm{\costmatrix^{-1/2} \beta_2}}\right)^{q-1} \\
        &= \frac{\qnorm{\costmatrix^{-1/2} \beta_2}^q}{\qnorm{\costmatrix^{-1/2} \beta_2}^{q-1}} \\
        &= \qnorm{\costmatrix^{-1/2} \beta_2}.
    \end{align}
    This allows us to simplify the above inequality and arrive at the condition
    \begin{align}
        \frac{\lambda}{\tau^+} \twonorm{\beta_1} - \qnorm{\costmatrix^{-1/2} \beta_1} &\geq \beta_1^\top \left(\frac{\lambda}{\tau^+} \frac{\beta_2}{\twonorm{\beta_2}} - \costmatrix^{-1/2} \frac{\partial}{\partial v} \qnorm{v}\big|_{v = \costmatrix^{-1/2} \beta_2} \right).
    \end{align}
    Now choose $\beta_1 = c v_d$ for some scalar $c \neq 0$. Then the RHS vanishes and $f(\beta)$ is convex only if
    \begin{align}
        &\frac{\lambda c}{\tau^+} - \frac{c}{\sigma_d} \geq 0\\
        \iff &\lambda \geq \frac{\tau^+}{\sigma_d} = \tau^+ \twonorm{\costmatrix^{-\nicefrac{1}{2}}}.
    \end{align} 
    Thus we've proven the required lower bound on $\lambda$. What remains is to show that this condition applies at a location in parameter space where the negatively labeled samples do not contribute to the gradient, and where the losses on the positively labeled samples are dominated by the second term of the maximum.
    
    To do this, we scale down $\beta_2 \to 0$ and choose the bias as $\beta_0 = -(1 + \uinv \dualnorm{\beta_2} + \epsilon (X + 1))$ for some very small positive $\epsilon$. It follows that for all $x$,
    \begin{align}
        \beta_2^\top x + \beta_0 + \uinv \dualnorm{\beta_2} &\leq \epsilon X -(1 + \uinv \dualnorm{\beta_2} + \epsilon (X + 1)) + \uinv \dualnorm{\beta_2} \\
        &\leq -(1+\epsilon).
    \end{align}
    This accomplishes both desiderata: first, it ensures that the loss on all negatively labeled points is $\max\{0, 1 + \beta^\top x + \beta_0 + \uinv \dualnorm{\beta}\} \leq \max\{0, -\epsilon\} = 0$, with gradient equal to $0$. Second, it ensures that on the positive examples, the second term in the loss dominates. We can see this by observing that the second term being larger is equivalent to
    \begin{align}
        0 &< 1 - \left( \beta_2^\top x + \beta_0 + \uinv \dualnorm{\beta_2}\right),
    \end{align}
    and by construction the RHS is at least $2+\epsilon$ for all $x$.
\end{proof}

As mentioned in the main body, we can also show the following additional result: for any $\lambda$ which induces convexity with $\twonorm{\beta}$, it must hold that $\dualnorm{\beta} \leq \twonorm{\beta}$, and thus it imposes the least regularization.

\begin{proposition}
    If the objective $\shingepoprisk^c(\beta) + \lambda \uinv \twonorm{\beta}$ is convex, then the dual regularization is no greater than the $\ell_2$ regularization: $\tau^+ \dualnorm{\beta} \leq \lambda \twonorm{\beta}$.
\end{proposition}
\begin{proof}
    Recall \cref{eq:beta-convexity} in the proof above: $f$ is convex only if for all vectors $\beta_1, \beta_2 \in \mathcal{B}$,
    \begin{align}
        \frac{\lambda}{\tau^+} ( \twonorm{\beta_1} - \twonorm{\beta_2} ) - \qnorm{\costmatrix^{-1/2} \beta_1} + \qnorm{\costmatrix^{-1/2} \beta_2} &\geq \left(\frac{\lambda}{\tau^+} \frac{\beta_2}{\twonorm{\beta_2}} - \costmatrix^{-1/2} \left( \frac{\partial}{\partial v} \qnorm{v}\big|_{v = \costmatrix^{-1/2} \beta_2} \right) \right)^\top (\beta_1 - \beta_2).
    \end{align}
    Choosing $\beta_1 = -\beta_2$ and simplifying, we get $\lambda ||\beta_2||_2 \geq \tau^+ ||\beta_2||_*$
\end{proof}

\section{Proofs of Lemmas in Main Body}
\label{appsec:main-lemmas-proofs}

\subsection{Proofs of \cref{lemma:dualnorm-soln,cor-zo-shinge}}
\begin{lemma}
    For any cost $c(x, x') = \phi(\psdnorm{x-x'}{\costmatrix})$, the maximum change to a user's prediction score that can result from strategic behavior is given by
    \begin{align}
    \beta^\top x(\beta) - \beta^\top x \le \uinv \dualnorm{\beta}
    \end{align}
    where $\dualnorm{\beta} := \dualpsdnorm{\beta}{\costmatrix} = \sup_{\psdnorm{v}{\costmatrix}=1} \beta^\top v$ is the \emph{$\costmatrix$-transformed dual norm} of $\beta$ and $\uinv := \sup r \in \R_{\geq 0} \textrm{ s.t. } \phi(r) \leq u$.
\end{lemma}
\begin{proof}
    A user at $x$ will move so as to maximize the inner product $\beta^\top x(\beta)$ so long as the cost of this move does not exceed the additional utility $u$ (and only up until the point that they achieve a positive classification). In other words, the maximum logit they will feasibly achieve is given by the optimization problem
    \begin{align}
        \sup_{x'} \beta^\top x' \quad\textrm{s.t. } c(x, x') \leq u.
    \end{align}
    We can reparameterize $x' = x + \delta$ to rewrite the objective as
    \begin{align}
        \beta^\top x + \sup_{\{\delta \;:\; \phi(\psdnorm{\delta}{\costmatrix}) \leq u\}} \beta^\top \delta,
    \end{align}
    which, recalling the definition of $\uinv$ and monotonicity of $\phi$, is equal to
    \begin{align}
        \beta^\top x + \sup_{\{\delta \;:\; \psdnorm{\delta}{\costmatrix} \leq \uinv\}} \beta^\top \delta
    \end{align}
    Here we recognize the variational formula for the dual norm, giving the solution $\beta^\top x + \uinv \dualnorm{\beta}$. \qedhere
\end{proof}

\begin{lemma}
    For any cost $c\in\costset$, $R_{0-1}^{\costfunc}(\beta) \leq \shingepoprisk^c(\beta)$.
\end{lemma}
\begin{proof}
    For a fixed sample $(x, y)$, recall the loss definitions:
    \begin{align}
        \zoloss^{\costfunc}(\beta) &:= \mathbf{1}\{\sign(\beta^\top (x+\delta)) \neq y\} \\
        \hingeloss^{\costfunc}(\beta) &:= \max\left(0, 1 - y \beta^\top (x+\delta) \right) \\
        \shingeloss^{\costfunc}(\beta) &:= \max\left(0, 1 - y (\beta^\top x + \uinv \dualnorm{\beta}) \right).
    \end{align}
    Since strategic response is agnostic to the loss used (i.e., $\delta$ does not change) and the hinge loss upper bounds the 0-1 loss, it is immediate that $\zoloss^{\costfunc} \leq \hingeloss^{\costfunc}$. Consider any point with true label $y=1$. If the point is positively classified (whether it moves or not) then $\zoloss^c = 0 \leq \shingeloss^c$. If the point is negatively classified and does not move, this means $\beta^\top x < -\uinv \dualnorm{\beta} \implies \beta^\top x+\uinv \dualnorm{\beta} < 0$, and therefore $\zoloss^c = 1 < \shingeloss^c$. So the claim holds for any point with $y=1$. 
    
    Next, if a point with true label $y=-1$ does not move, then neither loss changes as a result of strategic response, which means the strategic hinge loss is no less than the regular hinge loss. It remains to prove the inequality for points with $y=-1$ which move in response to the classifier.
    By \cref{lemma:dualnorm-soln} the classifier's output after strategic response will increase by no more than $\uinv \dualnorm{\beta}$. We have
    \begin{align}
        \hingeloss^\costfunc(\beta) &= \hingeloss(\beta^\top (x+\delta), y=-1) \\
        &= \max\{0, 1 + \beta^\top (x+\delta)\} \\
        &\leq \max\{0, 1+ \beta^\top x + \uinv \dualnorm{\beta}\} \\
        &= \shingeloss(\beta; \uinv \dualnorm{\beta}) \\
        &= \shingeloss^{\costfunc}(\beta). \qedhere
    \end{align}
\end{proof}

\subsection{Proof of \cref{lemma:solve-max}}
\begin{lemma}
    Fix some classifier $\beta$.
    Then for any dataset $(\gX\times\gY)^{n}$ and uncertainty set $\costset$,
    \textsc{MaxLossCost} runs in $\gO(nd + n\ln n)$ time and returns the value $k^*\in\R_{\geq 0}$ which maximizes the k-shifted strategic hinge loss $\shingeemprisk(\beta; k)$ subject to $k^* = \uinv \dualnorm{\beta}$
    for some cost $\cinc$.
\end{lemma}
\begin{proof}
    Recall the regularized strategic hinge loss $\shingeemprisk(\hat\beta; \uinv \dualnorm{\hat\beta}) = \frac{1}{n}\sum_{i=1}^{n} \max\{0, 1 - y_i(\hat\beta^\top x + \uinv \dualnorm{\hat\beta})\} + \lambda \uinv \dualnorm{\hat\beta}$. As this function depends on $\costfunc$ only through the dual norm, and since $\costset$ is a convex set and the norm is continuous, the worst-case cost \emph{scalar} can be reparameterized as $\argmax_{k\in[\dualnorm{\hat\beta}^{\min}, \dualnorm{\hat\beta}^{\max}]} \shingeemprisk(\hat\beta; \uinv k)$. This function is one-dimensional and piecewise linear in $k$, and therefore the maximum must occur either at an endpoint or at the boundary between two linear segments.

    By sorting the $v_i := y_i (1 - y_i \hat\beta^\top x_i)$, we get the values $1 - y_i \hat\beta^\top x_i$ with $y = +1$ in increasing order and those with $y = -1$ in decreasing order. At each step, we maintain the condition that for all $j' < j$, $v_{j'} - \uinv k \leq 0$. It follows that by increasing $k$ to the boundary of the next linear segment at $k'$, there are exactly $c_{+1}$ points for which the loss will decrease by $\uinv (k' - k)$ and $c_{-1}$ points for which the loss will increase by that same amount, while the regularization term increases by $\lambda \uinv (k' - k)$. Thus $r$ tracks the induced risk for the current $k$, and we keep track of the $k$ which so far induces the maximum risk. Finally, since we have moved to the next linear segment: either an example with $y = +1$ now has $0$ loss and will not change for the remainder; or an example with $y = -1$ has $>0$ loss and will contribute linearly to the risk for the remainder. We therefore update the appropriate count and iterate. The algorithm is complete when we reach the boundary $k = \dualnorm{\beta}^{\max}$. If before this point we reach the end of the sorted $v_j$, then we know that for the remaining possible increase $\dualnorm{\beta}^{\max} - k$, only the loss on the negative examples will change, growing linearly until the boundary. So we do one last evaluation and return the maximum.
\end{proof}

\section{Proof of Rademacher Generalization Bound}
\label{appsec:rademacher-proof}
\begin{theorem}[Strategic Hinge Generalization Bound]
\label{thm:rademacher-bound}
Fix a norm $\norm{\cdot}$. Assume $\max_{x\in\gD} \norm{x} \leq X$ and $\twonorm{\beta},\dualnorm{\beta} \leq B,\ \forall \beta\in\gB, \cinc$. Then with probability $\geq 1-\delta$, for all $\beta \in\gB$ and all cost functions $\cinc$,
\begin{align}
    R_{0-1}^c(\beta) &\leq \shingeemprisk^c(\beta) + \frac{B(4X + \uinv) + 3\sqrt{\ln \nicefrac{1}{\delta}}}{\sqrt{n}}
\end{align}
\end{theorem}
\cref{cor-zo-shinge} shows that $R_{0-1}^c(\beta) \leq \shingepoprisk^c(\beta)$. The result then follows from standard Rademacher bounds, requiring only the following additional Lemma:
\begin{lemma}
For any set of $n$ samples,
\begin{align}
    \hat\calR_n(\ell_{\textrm{s-hinge}} \circ \gB) &\leq \frac{B(4X + \uinv)}{2\sqrt{n}}.
\end{align}
\end{lemma}
\begin{proof}
Define the function class $\mathcal{H} := \{x \mapsto \beta^\top x + z(y) \uinv \dualnorm{\beta}\}$ ($z(y)$ follows notation from \citet{levanon2022generalized}---in our case it is always equal to 1 but more generally we let it be a map from $\{\pm 1\} \mapsto \{\pm 1\}$). With the definition of Rademacher complexity,
\begin{align}
    \hat\calR_n(\mathcal{H}) &= \E_\sigma \left[ \sup_{\beta\in\gB,\cinc} \frac{1}{n} \sum_{i=1}^n \sigma_i \cdot (\beta^\top x + z(y) \uinv \dualnorm{\beta}) \right] \\
    &\leq \E_\sigma \left[ \sup_{\beta\in\gB} \frac{1}{n} \sum_{i=1}^n \sigma_i \beta^\top x \right] + \E_\sigma \left[ \sup_{\beta\in\gB,\cinc} \frac{1}{n} \sum_{i=1}^n \sigma_i z(y) \uinv \dualnorm{\beta} \right].
\end{align}
The first term is the empirical Rademacher complexity of norm-bounded linear functions which is well known to be upper bounded by $\frac{2BX}{\sqrt{n}}$. An error in the proof by \citet{levanon2022generalized} dropped the second term from the calculation of the Rademacher complexity. We observe that it is not zero, but we can bound it as follows:

Since $z(y) = \pm 1$ it can be dropped due to the symmetry of the Rademacher variables (since the $y$ are fixed). Also, if the sum of the Rademacher variables is negative, the supremizing $\beta$ will have dual norm 0, and if the sum is positive, it will have dual norm $B$. Thus,
\begin{align}
    \E_\sigma \left[ \sup_{\beta\in\gB,\cinc} \frac{1}{n} \sum_{i=1}^n \sigma_i \uinv \dualnorm{\beta} \right] &= \P\left(\sum \sigma_i > 0 \right) \frac{\uinv B}{n} \E_\sigma \left[  \sum \sigma_i | \sum \sigma_i > 0 \right] \\
    &= \frac{\uinv B}{2n} \E_\sigma \left[ \left| \sum \sigma_i \right| \right] \\
    &\leq \frac{\uinv B}{2n} \sqrt{\E_\sigma \left[ \left( \sum \sigma_i \right)^2 \right]} \\
    &= \frac{\uinv B}{2\sqrt{n}},
\end{align}
where the second equality is due to the symmetry of the distribution over $\sigma$ and the inequality is by Jensen's. Since the function class $\ell_{\textrm{s-hinge}} \circ \gB$ is generated by a 1-Lipschitz function applied to $\mathcal{H}$, the claim follows by Talagrand's contraction lemma. \qedhere
\end{proof}

\section{Proof for Full-Batch Subgradient Method}
\label{appsec:subgradient-proof}

\begin{theorem}
    Suppose we run the subgradient method on the regularized k-shifted strategic hinge loss as described in \cref{alg:subgradient} for $T$ iterations and get classifier $\hat\beta$. Then with probability $\geq 1-\delta$, the worst-case 0-1 strategic loss under costs in $\costset$ can be bounded by
    \begin{align}
        \max_{\cinc} R_{0-1}^{\costfunc}(\hat\beta) &\leq \max_{\cinc} \shingeemprisk^c(\hat\beta) + \frac{B(4X + \uinv) +  3 \sqrt{\ln\nicefrac{2}{\delta}})}{\sqrt{n}}.
    \end{align}
    Furthermore, the sub-optimality of $\hat\beta$ with respect to the population minimax solution is bounded by
    \begin{align}
    	\max_{\cinc} \shingeemprisk^c(\hat\beta) &\leq \min_{\beta} \max_{\cinc} \shingepoprisk^{\costfunc}(\beta) + B \left(\frac{L}{\sqrt{T}} + (X+\uinv) \sqrt{\frac{\ln\nicefrac{2}{\delta}}{2n}} \right).
    \end{align}
\end{theorem}
\begin{proof}
    The first statement follows immediately with probability $\geq 1 - \nicefrac{\delta}{2}$ from \cref{thm:rademacher-bound} since the bound holds uniformly for all $\cinc$. Now we prove the second statement also holds with probability $\geq 1 - \nicefrac{\delta}{2}$, which we then combine via union bound. By Danskin's theorem, the function $r(\beta) := \max_{c \in \costset} \shingeemprisk^c(\beta)$ is convex in $\beta$, and its subgradient is defined by $\partial_\beta \shingeemprisk^{c^*}(\beta)$ where $c^* := \argmax_{c \in \costset} \shingeemprisk^c(\beta)$. 
    
    A standard result says that if we run the subgradient method on $r(\beta)$ with step size $\eta = \frac{\epsilon}{L^2}$ for $T \geq \frac{B^2 L^2}{\epsilon^2}$ steps, we will have $r(\beta^{t^*}) - \min_{\beta} r(\beta) \leq \epsilon$, which matches the hyperparameters of \cref{alg:subgradient} with $\epsilon = \frac{LB}{\sqrt{T}}$. The descent lemma in our setting is a bit different: we apply it to the $\costmatrix_{\min}$-transformed gradient and show convergence in the norm $\psdnorm{\cdot}{\costmatrix_{\min}^{-1}}$. That is, our update is $\beta^{(t+1)} = \beta^{(t)} - \eta \costmatrix_{\min} g_t$, and therefore
    \begin{align}
        \psdnormsq{\beta^{(t+1)} - \beta^*}{\costmatrix_{\min}^{-1}} &= \psdnormsq{\beta^{(t)} - \beta^*}{\costmatrix_{\min}^{-1}} - 2\eta g_t^\top (\beta^{(t)} - \beta^*) + \eta^2 \psdnormsq{g_t}{\costmatrix_{\min}} \\
        &\leq \psdnormsq{\beta^{(t)} - \beta^*}{\costmatrix_{\min}^{-1}} - 2\eta \left(\shingeemprisk(\beta^{(t)}) - \shingeemprisk(\beta^*)\right) + \eta^2 \twonormsq{\costmatrix_{\min}^{\nicefrac{1}{2}}g_t}.
    \end{align}
    Unrolling the chain of inequalities as in the typical convergence proof yields the same result, with the difference being that the Lipschitz constant will now be an upper bound on $\twonorm{\costmatrix_{\min}^{\nicefrac{1}{2}}g_t}$. Here we observe that for all $\costmatrix\in\costset$ (and therefore for every max player cost choice $\costmatrix_t$ at each iteration),
    \begin{align}
        \twonorm{\costmatrix_{\min}^{\nicefrac{1}{2}}g_t} &\leq \biggl \| \costmatrix_{\min}^{\nicefrac{1}{2}} \left(x + \uinv (1+\lambda) \frac{\partial \dualpsdnorm{\beta}{\costmatrix_t}}{\partial \beta} \right)  \biggr \|_2 \\
        &\leq \twonorm{\costmatrix_{\min}^{\nicefrac{1}{2}} x} + \uinv (1+\lambda) \biggl \| \costmatrix_{\min}^{\nicefrac{1}{2}} \frac{\partial \dualnorm{\costmatrix_t^{-\nicefrac{1}{2}}\beta}}{\partial \beta} \biggr \|_2 \\
        &\leq X + \uinv (1+\lambda) \biggl \| \costmatrix_{\min}^{\nicefrac{1}{2}} \costmatrix_t^{-\nicefrac{1}{2}} \frac{\partial \dualnorm{\beta}}{\partial \beta} \Bigr|_{\beta=\costmatrix_t^{-\nicefrac{1}{2}} \beta} \biggr \|_2 \\
        &\leq X + \uinv (1+\lambda) \overbrace{\twonorm{\costmatrix_{\min}^{\nicefrac{1}{2}} \costmatrix_t^{-\nicefrac{1}{2}}}}^{\leq 1} \biggl \| \frac{\partial \dualnorm{\beta}}{\partial \beta} \Bigr|_{\beta=\costmatrix_t^{-\nicefrac{1}{2}} \beta} \biggr \|_2 \\
        &\leq X + \uinv (1+\lambda) L_* =: L,
    \end{align}
    where $L_*$ is the Lipschitz constant of the gradient of the dual norm (recall for $p$-norms this is equal to $\max\left(1, d^{\nicefrac{(p-2)}{2p}}\right)$).
    Now continuing with the standard proof of convergence of the subgradient method gives us the desired result.
    
    It remains to show that \cref{alg:subgradient} successfully identifies the subgradient of $r$. We do so by invoking \cref{lemma:solve-max}, which says that we can efficiently solve for $c^*$ at each iteration. Thus we have that 
    \begin{align}
    	\max_{\cinc} \shingeemprisk^c(\hat\beta) &\leq \min_{\beta} \max_{\cinc} \shingeemprisk^{\costfunc}(\beta) + \frac{LB}{\sqrt{T}},
    \end{align}
    noting that this bound is with respect to the minimax \emph{empirical} risk. To complete the bound with respect to the minimax population risk, define $\beta^* := \argmin_\beta \max_{\cinc} \shingepoprisk^c(\beta)$ as the population minimax solution. Then we have with probability $\geq 1-\nicefrac{\delta}{2}$
    \begin{align}
        \min_{\beta} \max_{\cinc} \shingeemprisk^{\costfunc}(\beta) &\leq \max_{\cinc} \shingeemprisk^{\costfunc}(\beta^*) \\
        &\leq \max_{\cinc} \shingepoprisk^{\costfunc}(\beta^*) + B (X+\uinv) \sqrt{\frac{\ln\nicefrac{2}{\delta}}{2n}} \\
        &= \min_{\beta} \max_{\cinc} \shingepoprisk^{\costfunc}(\beta) + B (X+\uinv) \sqrt{\frac{\ln\nicefrac{2}{\delta}}{2n}},
    \end{align}
    where in the second inequality we've applied Hoeffding's between the empirical and population adversarial risk of $\beta^*$ (which is bounded in $[0, B(X+\uinv)]$) because that classifier does not depend on the training data.
    \qedhere
\end{proof}

\begin{algorithm}
\caption{Subgradient method on k-shifted strategic hinge loss}
\label{alg:subgradient}
\begin{algorithmic}
\STATE \textbf{input:} Dataset $\gD = \{(x_i, y_i)\}_{i=1}^n$, Iteration number $T$, Cost uncertainty set $\costset$. Upper bound $\uinv$, Regularization parameter $\lambda$.
\STATE \textbf{define:}\; $\beta^{(1)} \gets \mathbf{0}$.
\STATE \qquad\quad\ \ $\eta \gets \frac{B}{L\sqrt{T}}$.
\STATE \qquad\quad\ \ $\costmatrix_{\min} \gets \textrm{diag}([\mincostscale{1}, \ldots, \mincostscale{d}])$.
\FOR{$t = 1,\ldots,T$}
\STATE 1. $c_t \gets \textsc{MaxLossCost}(\gD, \beta^{(t)}, \costset, \uinv, \lambda)$.
\STATE 2. Choose $g_t \in \partial_\beta [ \shingeemprisk^{c_t}(\beta^{(t)}) + \lambda \uinv \dualnorm{\beta} ]$.
\STATE 3. $\beta^{(t+1)} \gets \beta^{(t)} - \eta \costmatrix_{\min} g_t$
\ENDFOR
\STATE \textbf{return} $\beta^{(t^*)}$, where $t^* := \argmin_t \shingeemprisk^{c_t}(\beta^{(t)}) + \lambda \uinv \dualnorm{\beta^{(t)}}$
\STATE
\STATE \textbf{subprocedure} \textsc{MaxLossCost}$(\gD, \beta, \costset, \uinv, \lambda)$
\begin{ALC@g}
    \STATE \textbf{define:} $\dualnorm{\beta}^{\min} := \min_{\cinc} \dualnorm{\beta},\ \dualnorm{\beta}^{\max} := \max_{\cinc} \dualnorm{\beta}$.
    \STATE \textbf{initialize:} $k \gets \dualnorm{\beta}^{\min}$, $k_{\max} \gets k$.
    \STATE \textbf{initialize:} $r \gets \shingeemprisk(\beta; \uinv k) + \lambda \uinv k$, $r_{\max} \gets r$.
    \STATE 1. Evaluate $v_i = y_i (1 - y_i \beta^\top x_i)$ for all $x_i$.
    \STATE 2. Sort $v_i$ in increasing order to get sorted indices $j$.
    \STATE 3. Initialize index $j \gets \min j\ \text{s.t. } v_j - \uinv k > 0$.
    \STATE 4. Initialize counts $c_{+1} \gets |\{j' \geq j : y_{j'} = +1\}|, c_{-1} \gets |\{j' < j : y_{j'} = -1\}|$.
    \WHILE{$k < \dualnorm{\beta}^{\max}$ \;\&\&\; $j < n$}
        \STATE $k' \gets \min\{\nicefrac{v_j}{\uinv}, \dualnorm{\beta}^{\max}\}$.
        \STATE $r \gets r + \uinv (k' - k) \left(\lambda + \frac{c_{-1} - c_{+1}}{n} \right)$.
        \STATE $k \gets k'$.
        \IF{$r > r_{\max}$}
            \STATE $r_{\max} \gets r$, $k_{\max} \gets k$.
        \ENDIF
        \STATE $c_{y_j} \gets c_{y_j} - y_j$.
        \STATE $j \gets j + 1$.
    \ENDWHILE
    \STATE
    \STATE \#\ \ \emph{Found maximizing norm scalar, now need matrix $\costmatrix$ which induces it.}
    \IF{$j == n$ \;\&\&\; $r + \uinv \left[\dualnorm{\beta}^{\max} - k \right] \left(\lambda + \frac{n^-}{n} \right) > r_{\max}$}
        \STATE \textbf{return} $\argmax_{\cinc} \dualnorm{\beta} = \textrm{diag}([\mincostscale{1}, \ldots, \mincostscale{d}])$.
    \ELSE
        \STATE \textbf{initialize:} $\hat\sigma \gets [\maxcostscale{1},\ldots,\maxcostscale{d}]$.
        \FOR{$i = 1,\ldots,d$}
            \STATE Let $\hat\costmatrix := \textrm{diag}([\hat\sigma_1, \ldots, \mincostscale{i}, \ldots, \hat\sigma_d])$.
        \IF{$\dualpsdnorm{\beta}{\hat\costmatrix} < k_{\max}$}
            \STATE $\hat\sigma_i \gets \mincostscale{i}$.
            \STATE \textbf{continue}.
        \ENDIF
        \STATE Let $\hat\costmatrix_{i=0} := \textrm{diag}([\hat\sigma_1, \ldots, 0_i, \ldots, \hat\sigma_d])$.
        \STATE $\hat\sigma_i \gets \frac{|\beta_i|}{\left(k_{\max}^p - \dualnorm{\hat\costmatrix_{i=0}^{-\nicefrac{1}{2}}\beta}^p\right)^{\nicefrac{1}{p}}}$.
        \STATE \textbf{return} $\textrm{diag}(\hat\sigma)$.
        \ENDFOR
    \ENDIF
\end{ALC@g}
\end{algorithmic}
\end{algorithm}

\clearpage
\section{Proof for Stochastic Mirror Descent-Ascent}
\label{appsec:mirror-proofs}

We let $0 < \epsilon \leq 1$ denote the discretization parameter which tunes the size of the set $|\gS|$ (for simplicity, assume $\nicefrac{1}{\epsilon}$ is an integer). Specifically, we choose $\nicefrac{1}{\epsilon}$ equally spaced points in each dimension's range of \emph{inverse} eigenvalues and then define the elements of $\gS$ to be the collection of smallest values in each dimension, then all the second-smallest values, etc. In this way we discretize the ``diagonal'' of the cost uncertainty set $\costset$ to avoid an exponential dependence on the dimension.

\begin{theorem}
    Suppose we run SMDA on the regularized strategic hinge loss as described in \cref{alg:mirror-descent} for $T$ iterations and get averaged classifier iterates $\tilde\beta$. Define the convergence error
    \begin{align}
        \varepsilon_T &:= \max_{\cinc} \shingepoprisk^{\costfunc}(\tilde\beta) - \min_{\beta} \max_{\cinc} \shingepoprisk^{\costfunc}(\beta).
    \end{align}
    Then over the randomness of the optimization procedure it holds that
    \begin{align}
        \E[\varepsilon_T] \lesssim B \left[ \uinv |1 - \lambda| \max_i \sqrt{\epsilon \left( \mincostscale{i}^{-2} - \maxcostscale{i}^{-2} \right)} +
        \frac{L + (B^{-1} + X + \uinv) \sqrt{\ln \nicefrac{1}{\epsilon} }}{\sqrt{T}}\;\right].
    \end{align}
\end{theorem}
\begin{proof}
    From \cref{thm:cvx-regularization1}, we have that the regularized loss is convex in $\beta$. However, the loss is \emph{not} concave in any parameterization of the cost function $c$. To resolve this, we discretize the space of cost functions. We parameterize the eigenvalues of the inverse cost matrix $\costmatrix^{-1}$ as a choice of eigenvalue for each eigenvector $v_i$ from the compact set $[\maxcostscale{i}^{-2}, \mincostscale{i}^{-2}]$. We do so by discretizing the space of choices linearly as $\sigma_{k}^{-2} = \maxcostscale{i}^{-2} + k \epsilon (\mincostscale{i}^{-2} - \maxcostscale{i}^{-2})$ for $k\in[1,  \nicefrac{1}{\epsilon} ]$. Now we can instead optimize
    \begin{align}
        \min_{\beta} \max_{k \in [ \nicefrac{1}{\epsilon} ]} \shingepoprisk^{\costfunc(k)}(\beta),
    \end{align}
    where $\shingepoprisk^{\costfunc(k)}$ denotes the loss under the cost defined by the eigenvalues induced by $k$. As this is a discrete set of $\nicefrac{1}{\epsilon} $ choices, this objective is exactly equivalent to
    \begin{align}
        \min_{\beta} \max_{\delta\in\Delta^{1/\epsilon}} \sum_{i=1}^{\nicefrac{1}{\epsilon}} \delta_i \shingepoprisk^{\costfunc(i)}(\beta),
    \end{align}
    where $\Delta^{1/\epsilon}$ is the simplex over $\nicefrac{1}{\epsilon}$ values such that $\delta_i \geq 0\ \forall i,\  \sum_i \delta_i = 1$. This objective \emph{is} concave in $\delta$, which means it can be solved via SMDA as described by \citet{nemirovski2009robust}. Let $\tilde\beta, \tilde \delta$ denote the two players' averaged iterates over choices $\beta, \delta$. Define
    \begin{align}
        \hat\varepsilon_T &:= \max_{k\in [\nicefrac{1}{\epsilon}]} \shingepoprisk^{\costfunc(k)}(\tilde\beta) - \min_{\beta} \sum_{i=1}^{\nicefrac{1}{\epsilon}} \tilde\delta_i \shingepoprisk^{\costfunc(i)}(\beta).
    \end{align}
    Note that this is the expected sub-optimality gap for a new optimization problem, whose solution is not the same as the one in the theorem statement.
    \citet{nemirovski2009robust} prove that after $T$ iterations of SMDA with the appropriate step size we have (in their original notation)
    \begin{align}
        \E[\hat\varepsilon_T] &\leq 2\sqrt{\frac{10 [R_x^2 M_{*, x}^2 + M_{*, y}^2 \ln m]}{N}}.
    \end{align}
    Translating these terms into our notation,
    \begin{align}
        m &= \nicefrac{1}{\epsilon}, \\
        N &= T, \\
        M_{*, x}^2 &= \max_{1\leq i \leq \nicefrac{1}{\epsilon}} \E[\normsq{\nabla \shingeloss^{\costfunc(i)}(\tilde\beta)}] \leq L^2, \\
        M_{*, y}^2 &= \E\left[ \max_{1\leq i \leq \nicefrac{1}{\epsilon}} |\shingeloss^{\costfunc(i)}(\tilde\beta)|^2 \right] \leq (1 + B (X + \uinv))^2, \\
        R_x^2 &= \frac{1}{2}\max_{b_1, b_2 \in \gB} \twonormsq{b_1 - b_2} \lesssim B^2,
    \end{align}
    where the last inequality follows from the triangle inequality and the upper bound on $\twonorm{\beta}$.
    Plugging these in gives the bound
    \begin{align}
        \E[\hat\varepsilon_T] &\lesssim \sqrt{ \frac{ B^2 L^2 + (1 + B (X + \uinv))^2 \ln \nicefrac{1}{\epsilon}}{T}} \\
        &\lesssim B\;\frac{L + (B^{-1} + X + \uinv) \sqrt{\ln \nicefrac{1}{\epsilon}}}{\sqrt{T}}.
    \end{align}
    Next, we can rewrite the convergence error in the theorem statement in terms of this error as
    \begin{align}
        \varepsilon_T &:= \max_{\cinc} \shingepoprisk^{\costfunc}(\tilde\beta) - \min_{\beta} \max_{\cinc} \shingepoprisk^{\costfunc}(\beta) \\
        &= \hat\varepsilon_T + \left( \max_{\cinc} \shingepoprisk^{\costfunc}(\tilde\beta) - \max_{k\in [\nicefrac{1}{\epsilon}]} \shingepoprisk^{\costfunc(k)}(\tilde\beta) \right) \\
        &\qquad \qquad- \overbrace{\left( \min_{\beta} \max_{\cinc} \shingepoprisk^{\costfunc}(\beta) - \min_{\beta} \sum_{i=1}^{\nicefrac{1}{\epsilon}} \tilde\delta_i \shingepoprisk^{\costfunc(i)}(\beta) \right)}^{\geq 0} \\
        &\leq \hat\varepsilon_T + \left( \max_{\cinc} \shingepoprisk^{\costfunc}(\tilde\beta) - \max_{k\in [\nicefrac{1}{\epsilon}]} \shingepoprisk^{\costfunc(k)}(\tilde\beta) \right).
    \end{align}
    This last term represents the error due to discretization. Revisiting the regularized risk definition, for any $\costfunc$ and $k$ we have
    \begin{align}
        \shingepoprisk^{\costfunc}(\tilde\beta) - \shingepoprisk^{\costfunc(k)}(\tilde\beta) &= \E\biggl[ \max\{0, 1 - y(\tilde\beta^\top x + \uinv \dualnorm{\costmatrix_c^{-1/2} \tilde\beta})\} \\
        &\quad - \max\{0, 1 - y(\tilde\beta^\top x + \uinv \dualnorm{\costmatrix_{c(k)}^{-1/2} \tilde\beta})\} \biggr] \\
        &\qquad + \lambda \uinv (\dualnorm{\costmatrix_c^{-1/2} \tilde\beta} - \dualnorm{\costmatrix_{c(k)}^{-1/2} \tilde\beta}) \\
        &\leq \uinv |1 - \lambda| \left|\dualnorm{\costmatrix_{c(k)}^{-1/2} \tilde\beta} -\dualnorm{\costmatrix_c^{-1/2} \tilde\beta} \right| \\
        &\leq \uinv |1 - \lambda| B \cdot \sigma_{\max}\left(\costmatrix_{c(k)}^{-1/2} - \costmatrix_c^{-1/2}\right)
    \end{align}
    by the reverse triangle inequality. Since these two matrices have the same eigenvectors, the maximum eigenvalue of their difference is simply the maximum absolute difference between their respective eigenvalues. By construction, for any choice $\costmatrix_c^{-1}$, there is a choice $k \in [\nicefrac{1}{\epsilon}]$ which differs in spectrum by no more than $\epsilon \cdot \max_i (\mincostscale{i}^{-2} - \maxcostscale{i}^{-2})$ in any given direction, and therefore we have
    \begin{align}
        \max_{\cinc} \shingepoprisk^{\costfunc}(\tilde\beta) - \max_{k\in [\nicefrac{1}{\epsilon}]} \shingepoprisk^{\costfunc(k)}(\tilde\beta) &\leq \uinv B \cdot |1 - \lambda| \cdot \max_i \left| \sqrt{\sigma_{i}\left(\costmatrix_{c(k)}^{-1} \right)} - \sqrt{\sigma_{i}\left(\costmatrix_{c}^{-1} \right)} \right| \\
        &\leq \uinv B \cdot |1 - \lambda| \cdot \max_i  \sqrt{\left|\sigma_{i}\left(\costmatrix_{c(k)}^{-1} \right) - \sigma_{i}\left(\costmatrix_{c}^{-1} \right)\right|} \\
        &\leq \uinv B \cdot |1 - \lambda| \cdot \max_i \sqrt{\epsilon \left(\mincostscale{i}^{-2} - \maxcostscale{i}^{-2} \right)}.
    \end{align}
    Combining this bound with the one above on $\hat\varepsilon_T$ and taking expectations gives the result.
\end{proof}


\begin{corollary}
Recall $D := \max_i (\mincostscale{i}^{-2} - \maxcostscale{i}^{-2})$. Choosing $\epsilon = \Theta\left( \frac{\ln T}{T \max\left( 1, D \right)} \right)$, we have
\begin{align}
    \E[\varepsilon_T] &\lesssim \frac{LB}{\sqrt{T}} + B(X+\uinv) \sqrt{\frac{\ln T + \max\left(0, \ln D \right)}{T}}.
\end{align}
\end{corollary}

\section{Goodhart's Law Under Known Costs}
\label{apx:in_principle}

We here show that when each user's strategic response to a classifier is known, then in principle (\emph{information theoretically,} discarding optimization concerns\nocite{rosenfeld2024almost}),
strategic response has no effect on predictive performance.

We start with a correspondence between classifiers operating on (non-strategic) inputs and appropriately modified classifiers operating on strategically modified inputs. In line with prior works---and unlike the other results in this paper---we make the additional assumption that $u$ is fixed and $\phi$ is strictly monotone (thus $\uinv = \phi^{-1}(u)$) We also explicitly parameterize the bias term in the classifiers because it plays an important role in the result.\looseness=-1
\begin{proposition}
Let $f(x)=\mathbf{1}\{\beta^\top x + b \geq 0\}$ be the prediction of the classifier parameterized by $(\beta, b)$, and let $f'(x)=\mathbf{1}\{\beta^\top x + b' \geq 0\}$ denote this classifier with a shifted bias $b' = b - \uinv \dualnorm{\beta}$.
Then for all $x$, it holds that
$f(x)=f'(x')$, where $x'=x(\beta,b')$ is the new location of $x$ after strategic response to the classifier $(\beta, b')$.
\end{proposition}
This results states that $f$ outputs on every data point $x$
exactly what $f'$ (which only differs in its shifted bias term) outputs on $x'$---which represents the exact same ``user'' \emph{after it has strategically responded}.
\begin{proof}
The proof is simple. We consider three separate cases:
\begin{itemize}
\item 
If $f(x)=0$ and $x$ is too far to cross the decision boundary of $f$ (so $x'(\beta, b) = x$), then since $\uinv \dualnorm{\beta} \geq 0$, $x$ is also too far to cross the decision boundary of $f'$, (that is, $x'(\beta, b') = x$). Therefore, $f'(x') = f'(x) = f(x)$.

\item 
If $f(x)=0$ and $x$ is close enough to cross the decision boundary (so $x'(\beta, b) \neq x$),
then \cref{lemma:dualnorm-soln} implies that the maximum amount by which $x$ will change its linear prediction under $\beta$ is $\uinv \dualnorm{\beta}$. Since $f(x) = 0 \implies 0 > \beta^\top x + b \implies -\uinv \dualnorm{\beta} > \beta^\top x + b'$, this means $x$ cannot force $f'(x') = 1$ without increasing the prediction by more than $\uinv \dualnorm{\beta}$, and thus it will not move. So, $f'(x') = f'(x) = f(x)$.

\item
If $f(x)=1$, then $0 \leq \beta^\top x + b \implies -\uinv \dualnorm{\beta} \leq \beta^\top x + b'$. Thus, either $x$ will already get a positive classification from $f'$, or it will be able to move enough to cross the decision boundary. Either way, $f'(x') = 1 = f(x)$. \qedhere
\end{itemize}
\end{proof}

Note that this proof applies even if we \emph{don't} know the user's cost function---it is sufficient to know for each user the maximum potential increase in their predicted logit under $\beta$ after strategic response, i.e. $\beta^\top (x'(\beta) - x)$. Applying this insight to the optimal classifier gives the following result:
\begin{corollary}
Fix some distribution $\gD$,
and let $f^*$ with parameters $(\beta^*,b^*)$ be the classifier minimizing the \emph{non-strategic} 0-1 risk on $\gD$. Denote this risk as $\alpha$.
Then for $f'$ with $(\beta^*,b^* - \uinv \dualnorm{\beta^*})$,
its expected \emph{strategic} 0-1 risk on $\gD$ is exactly $\alpha$.
\end{corollary}

The take-away is that
if we are able to minimize the standard 0-1 loss, and we know how users will respond (e.g. by knowing the exact cost function),
then a trivial modification would provide us with an optimal classifier for the strategic 0-1 loss. Thus, strategic response does not pose any additional statistical difficulty over standard classification.

Importantly, however, this result does \emph{not} imply that minimizing a \emph{proxy} loss (such as the hinge or logistic loss) on non-strategic data and applying the transformation would give a good strategic classifier;
this precisely why the strategic hinge loss is needed in the first place. Further, this result does not account for the social cost  of the classifier $f'$ versus some other classifier \citep{milli2019social}---there could be a different classifier with similar accuracy under strategic response that induces a smaller cost to the users.

\end{document}